\documentclass[11pt]{article}
\usepackage[in]{fullpage}
\usepackage[parfill]{parskip}
\usepackage{amsmath,amsthm,amssymb,bbm}
\usepackage{mathtools}
\usepackage{cases}
\usepackage{dsfont}
\usepackage{microtype}
\usepackage{subfigure}
\usepackage{algorithm,algorithmic}
\usepackage{color}
\usepackage{appendix}

\usepackage{bm}
\usepackage{subfigure}
\usepackage{algorithm,algorithmic}
\usepackage{color}
\usepackage{booktabs}       
\usepackage{appendix}
\usepackage{authblk}

\usepackage{url}
\usepackage[authoryear]{natbib}
\usepackage[colorlinks,citecolor=blue,urlcolor=blue,linkcolor=blue,linktocpage=true]{hyperref}
\usepackage{cleveref} 
\usepackage{enumitem}

\newtheorem{theorem}{Theorem}[section]
\newtheorem{lemma}[theorem]{Lemma}

\newtheorem{definition}[theorem]{Definition}

\newtheorem{remark}[theorem]{Remark}

\newtheorem{question}[theorem]{Question}

\definecolor{maroon}{RGB}{192,80,77}
\newcommand{\maroon}[1]{\textcolor{maroon}{#1}}

\begin{document}
	
	\title{Near-Optimal Reinforcement Learning with Self-Play \\ under Adaptivity Constraints}
	\author[1]{Dan Qiao}
	\author[1]{Yu-Xiang Wang}
	\affil[1]{Department of Computer Science, UC Santa Barbara}
	\affil[ ]{\texttt{danqiao@ucsb.edu}, \;
		\texttt{yuxiangw@cs.ucsb.edu}}
	
	\date{}
	
	\maketitle
	
	\begin{abstract}
		We study the problem of multi-agent reinforcement learning (MARL) with adaptivity constraints --- a new problem motivated by real-world applications where deployments of new policies are costly and the number of policy updates must be minimized. For two-player zero-sum Markov Games, we design a (policy) elimination based algorithm that achieves a regret of $\widetilde{O}(\sqrt{H^3 S^2 ABK})$, while the batch complexity is only $O(H+\log\log K)$. In the above, $S$ denotes the number of states, $A,B$ are the number of actions for the two players respectively, $H$ is the horizon and $K$ is the number of episodes. Furthermore, we prove a batch complexity lower bound $\Omega(\frac{H}{\log_{A}K}+\log\log K)$ for all algorithms with $\widetilde{O}(\sqrt{K})$ regret bound, which matches our upper bound up to logarithmic factors. As a byproduct, our techniques naturally extend to learning bandit games and reward-free MARL within near optimal batch complexity. To the best of our knowledge, these are the first line of results towards understanding MARL with low adaptivity.
	\end{abstract}

    \newpage
    \tableofcontents
    \newpage
	

\section{Introduction}\label{sec:intro}

This paper considers the problem of multi-agent reinforcement learning (multi-agent RL), where multiple agents aim to make decisions simultaneously in an unfamiliar environment to maximize their individual cumulative rewards. Multi-agent RL is used not only 
in large-scale strategy games like Go \citep{silver2017mastering}, Poker \citep{brown2019superhuman} and MOBA games \citep{ye2020towards}, but also in various real-world applications such as autonomous driving \citep{shalev2016safe}, household power management \citep{chung2020distributed}, and computer networking \citep{bhattacharyya2019qflow}. 

The sheer amount of computation needed for self-play-based learning in these applications often demands the algorithm to run in a distributed fashion where the communication cost becomes a bottleneck.
In such circumstances, instead of syncing up after each single trajectory, a more practical alternative is to assign a larger batch of work for each machine to perform independently and sync up only sporadically.
The need for infrequent communication could be hard constraints in applications such as autonomous driving \citep{shalev2016safe}. Deploying new policies to vehicle firmware takes weeks, while new data are being collected in millions of cars every second.
These constraints render standard multi-agent RL algorithms that require altering the policy after each new data point impractical.

In the scenarios discussed above, the agent needs to minimize the number of policy deployments while learning a good policy using
(nearly) the same number of samples as its fully-adaptive counterparts.
Previous works brought up different notions to measure the adaptivity of an online RL algorithm, including switching cost \citep{bai2019provably,zhang2020almost,qiao2022sample,gao2021provably,wang2021provably,he2023nearly,qiao2023logarithmic,kong2021online,velegkas2022best,zhao2023nearly,xiong2023general}, batch complexity \citep{perchet2016batched,gao2019batched,qiao2022sample,zhang2022near,wang2021provably,johnson2023sample,xiong2023general} and deployment efficiency \citep{matsushima2020deployment,huang2021towards,qiao2022near,modi2021model}. Although algorithms with low adaptivity have been designed for various Markov decision process (MDP) settings, all of the previous works focused on the single-agent setting while the solution to multi-agent RL with low adaptivity is still unknown. Therefore it is natural to question:

\begin{question}
Is it possible to design a self-play algorithm to solve Markov games while achieving near optimal adaptivity and sample complexity simultaneously? 
\end{question}
 
\begin{table*}[!t]\label{tab:comparison}
\centering
\resizebox{\linewidth}{!}{
\begin{tabular}{ |c|c|c|c| } 
\hline
\textit{Algorithms for Markov games} & \textit{Single-agent (B=1)?} & \textit{Regret}  & \textit{Batch complexity} \\
\hline 
VI-ULCB \citep{bai2020provable} & No & $\widetilde{O}(\sqrt{H^{3}S^2 ABT})$   & $K$ \\ 
Nash VI \citep{liu2021sharp} & No & $\widetilde{O}(\sqrt{H^{2}SABT})$  & $K$ \\ 
Algorithm 1 in \citet{zhang2022near} $^*$ & Yes & $\widetilde{O}(\sqrt{H^{2}SAT})$  & $O(H+\log\log K)$ \\  
\textcolor{blue}{Our Algorithm~\ref{alg:main} (Theorem \ref{thm:main})} &\textcolor{blue}{No} &\textcolor{blue}{$\widetilde{O}(\sqrt{H^2 S^{2}ABT})$}  & \textcolor{blue}{$O(H + \log\log K)$} \\
\hline
{Lower bound \citep{bai2020provable}} & No & $\Omega(\sqrt{H^2 S(A+B)T})$ & No constraints.    \\
\textcolor{blue}{Lower bound (Theorem~\ref{thm:lower})} & \textcolor{blue}{No} & \textcolor{blue}{if $\widetilde{O}(\sqrt{T})$ (``Optimal regret'')} & \textcolor{blue}{ $\Omega(\frac{H}{\log_A K}+\log\log K)$}    \\
\hline
\textit{Algorithms for bandit games} & \textit{Single-agent (B=1)?} & \textit{Regret}  & \textit{Batch complexity}  \\
\hline
BaSE \citep{gao2019batched} $^\dagger$ & Yes & $\widetilde{O}(\sqrt{AK})$  &  $O(\log\log K)$ \\  
\textcolor{blue}{Our Algorithm~\ref{alg:bandit} (Theorem \ref{thm:bandit})} & \textcolor{blue}{No} & \textcolor{blue}{$\widetilde{O}(\sqrt{ABK})$} & \textcolor{blue}{$O(\log\log K)$} \\
\hline
\textit{Algorithms for reward-free exploration} & \textit{Single-agent (B=1)?} & \textit{Sample (episode) complexity}  & \textit{Batch complexity}  \\
\hline
VI-Explore \citep{bai2020provable} & No & $\widetilde{O}(\frac{H^{5}S^{2}AB}{\epsilon^{2}})$  & $\widetilde{O}(\frac{H^{7}S^{4}AB}{\epsilon})^\ddagger$ \\   
LARFE \citep{qiao2022sample} $^\star$ & Yes & $\widetilde{O}(\frac{H^{5}S^{2}A}{\epsilon^{2}})$  & $O(H)$ \\ 
\textcolor{blue}{Our Algorithm~\ref{alg:rfe} (Theorem \ref{thm:rfe})} & \textcolor{blue}{No} &  \textcolor{blue}{$\widetilde{O}(\frac{H^{3}S^{2}AB}{\epsilon^{2}})$} & \textcolor{blue}{$O(H)$} \\
\hline 
\end{tabular}
}
\caption{
Comparison of our results (in \textcolor{blue}{blue}) to existing work regarding problem type, regret/sample complexity, and batch complexity. A ``Yes'' in the column ``Single-agent (B=1)?'' means that the work is specific to the single-agent case ($B=1$) and cannot be directly applied under the two-player game setting ($B>1$) as in this paper. We list such works here for comparison. In the above, $S$ denotes the number of states, $A,B$ are the number of actions for the two players respectively, $H$ is the horizon and $K$ is the number of episodes ($T=HK$ is the number of steps). $*$: This result is derived under the special case of single-agent MDP (Markov game with $B=1$). In this case, our Algorithm \ref{alg:main} achieves the same batch complexity and a regret bound sub-optimal by $\sqrt{S}$. $\dagger$: The result is derived under the batched multi-armed bandits setting (bandit games with $B=1$). In this case, our Algorithm \ref{alg:bandit} achieves the same guarantees as BaSE. $\ddagger$: For the first part of the algorithm, there are $\widetilde{O}(\frac{H^{7}S^{4}AB}{\epsilon})$ episodes of data collected using EULER, which can lead to the same number of batch complexity in the worst case.  $\star$: This result is derived under the special case of single-agent MDP. Directly applying our Algorithm \ref{alg:rfe} to the setting will lead to a significant improvement of $H^2$ in sample complexity.}
\end{table*}

\noindent\textbf{Our contributions.} In this paper, we answer the above question affirmatively by proposing a new low-adaptive algorithm (Algorithm~\ref{alg:main}). Furthermore, the framework of Algorithm \ref{alg:main} naturally adapts to the bandit game setting and the more challenging low adaptive reward-free setting. Our concrete contributions are summarized as follows. 

\begin{itemize}

\item We design a new policy elimination based algorithm (Algorithm~\ref{alg:main}) that achieves $O(H+\log \log K)$ batch complexity and $\widetilde{O}(\sqrt{H^2 S^{2}ABT})$ regret bound (Theorem~\ref{thm:main}). To our knowledge, this provides the first result under multi-agent RL with low adaptivity. Moreover, the regret bound has optimal dependence on $T$ while the batch complexity is optimal up to logarithmic factors among all algorithms with $\widetilde{O}(\sqrt{T})$ regret bound (due to our Theorem \ref{thm:lower}).

\item Under the bandit game setting (a special case of Markov game with $H=S=1$), Algorithm \ref{alg:bandit} achieves $O(\log\log K)$ batch complexity and $\widetilde{O}(\sqrt{ABK})$ regret (Theorem \ref{thm:bandit}). The batch complexity is optimal and the result strictly generalizes the best known result of batched multi-armed bandits \citep{gao2019batched}.   

\item We also propose a new low-adaptive algorithm (Algorithm~\ref{alg:rfe}) for reward-free exploration in Markov games. It comes with an optimal batch complexity of $O(H)$ and provably identifies an $\epsilon$-approximate Nash policy simultaneously for all possible reward functions (Theorem~\ref{thm:rfe}). The result improves over previous results in both sample complexity and batch complexity.
\end{itemize}

\noindent\textbf{Related work.} There is a large and growing body of literature on the statistical theory of reinforcement learning that we will not attempt to thoroughly review. Detailed comparisons with existing work on multi-agent RL  \citep{bai2020provable,liu2021sharp}, batched RL \citep{zhang2022near}, batched multi-armed bandits \citep{gao2019batched} and reward-free exploration \citep{bai2020provable,qiao2022sample} are given in Table~\ref{tab:comparison}. For more details about related works, please refer to Appendix~\ref{appr} and the references therein. Notably, all existing algorithms with low adaptivity focus on the single-agent case. In comparison, our results work for the more general multi-player setting, and thus can be considered as generalization of previous results.

A recent line of works provide non-asymptotic guarantees for learning Markov Games. \citet{bai2020provable} developed the first provably-efficient algorithms in MGs based on optimistic value iteration, whose result is improved by \citet{liu2021sharp} using model-based approach. Meanwhile, model-free approaches are shown to break the curse of multiagency and improve the dependence on action space \citep{bai2020near,jin2021v,mao2022improving,wang2023breaking,cui2023breaking}. Following works also extended the results to MGs with function approximation \citep{xie2020learning,huang2022towards,jin2022power,li2022learning}. However, all these works applied fully adaptive algorithms, which can be difficult to implement in practical scenarios. In comparison, our algorithms achieve near optimal adaptivity.

In this paper, we measure the adaptivity by batch complexity \citep{zhang2022near} which requires decisions about policy updates to be made at only a few predefined checkpoints.\footnote{The formal def. of batch complexity is deferred to Section \ref{sec:setup}.} Besides, there are other measurements of adaptivity. The most prevalent measurement is \emph{switching cost}, which measures the number of policy switches. However, previous works minimizing switching cost only allow usage of deterministic policies \citep{bai2019provably,zhang2020almost,qiao2022sample}. It is known that in many Markov games like Rock paper scissors, the Nash policy can only be stochastic and running deterministic policies will lead to a linear regret. Therefore, switching cost is not an appropriate measurement for the multi-agent case. \footnote{While it is true that we can generalize the definition of switching cost to support general (stochastic) policies, in this case the guarantee of batched RL is stronger due to its nature that the timestep to switch policy is predetermined.} Meanwhile, \citet{matsushima2020deployment} proposed the notion of \emph{deployment efficiency}, which is similar to batched RL with additional requirement that each policy deployment should have similar size. Deployment efficient RL is studied by some following works \citep{huang2021towards,qiao2022near,modi2021model}. However, as pointed out by \citet{qiao2022near}, deployment complexity is not a good measurement of adaptivity when studying regret minimization.
	

\section{Problem Setup}\label{sec:setup}
\noindent\textbf{Notations.} Throughout the paper, for $n\in\mathbb{Z}^{+}$, $[n]=\{1,2,\cdots,n\}$. For any set $U$, $\Delta(U)$ denotes the set of all possible distributions over $U$. In addition, we use standard notations such as $O$ and $\Omega$ to absorb constants while $\widetilde{O}$ and $\widetilde{\Omega}$ suppress logarithmic factors.

\noindent\textbf{Markov Games.} Markov Games (MGs) generalize the standard Markov Decision Processes (MDPs) into the multi-player setting, where each player aims to maximize her own reward. We consider \emph{two-player zero-sum} episodic Markov Games, denoted by a tuple $\mathcal{MG}=(H, \mathcal{S}, \mathcal{A}, \mathcal{B}, P_h, r_h)$, where $H$ is the horizon, $\mathcal{S}$ is the state space with $S:=|\mathcal{S}|$. $\mathcal{A}$ and $\mathcal{B}$ are the action space for the max-player (who aims to maximize the total reward) and the min-player (who aims to minimize the total reward) respectively, where $A:=|\mathcal{A}|,B:=|\mathcal{B}|$ are finite. The non-stationary transition kernel has the form $P_h:\mathcal{S}\times\mathcal{A}\times\mathcal{B}\times\mathcal{S} \mapsto [0, 1]$  with $P_{h}(s^{\prime}|s,a,b)$ representing the probability of transition from state $s$, action $(a,b)$ to next state $s^\prime$ at time step $h$. In addition, $r_h(s,a,b)$ denotes the known\footnote{Our results easily generalize to the case with stochastic reward.} expected (immediate) reward function. Without loss of generality, we assume each episode starts from a fixed initial state $s_{1}$.\footnote{The generalized case where the initial distribution is an arbitrary distribution can be recovered from this setting by adding one layer to the MG.} At time step $h\in[H]$, two players observe $s_h$ and choose their actions $a_h\in\mathcal{A}$ and $b_h\in\mathcal{B}$ at the same time. Then both players observe the action of their opponent and receive reward $r_h(s_h,a_h,b_h)$, the environment will transit to $s_{h+1}\sim P_h(\cdot|s_h,a_h,b_h)$. 

\noindent\textbf{Markov policy, value function.} A (Markov) policy $\mu$ of the max-player can be seen as a series of mappings $\mu=(\mu_1,\cdots,\mu_H)$, where each $\mu_h$ maps each state $s \in \mathcal{S}$ to a probability distribution over actions $\mathcal{A}$, \emph{i.e.} $\mu_h: \mathcal{S}\rightarrow \Delta(\mathcal{A})$. A Markov policy $\nu$ for the min-player is defined similarly. 

Given a pair of policies $(\mu,\nu)$ and $h\in[H]$, the value function $V^{\mu,\nu}_h(\cdot)$ and Q-value function $Q^{\mu,\nu}_h(\cdot,\cdot,\cdot)$ are defined as:
$
V^{\mu,\nu}_h(s)=\mathbb{E}_{\mu\times\nu}[\sum_{t=h}^H r_{t}|s_h=s] ,
Q^{\mu,\nu}_h(s,a,b)=\mathbb{E}_{\mu\times\nu}[\sum_{t=h}^H  r_{t}|s_h,a_h,b_h=s,a,b],\;\forall\, s,a,b\in\mathcal{S}\times\mathcal{A}\times\mathcal{B}.
$   
Then the Bellman equation follows $\forall\, h\in[H]$:
\begin{align*}
Q^{\mu,\nu}_h(s,a,b)&=[r_{h}+P_{h}V^{\mu,\nu}_{h+1}](s,a,b),\;\\
V^{\mu,\nu}_h(s)&=[\mathbb{E}_{\mu\times\nu}Q^{\mu,\nu}_h](s).
\end{align*}

In this work, we will consider different MGs with respective transition kernels and reward functions. We define the value function for policy $\pi=(\mu,\nu)$ under MG $(\widetilde{r},\widetilde{P})$ as below
$$V^{\pi}(\widetilde{r},\widetilde{P})=\mathbb{E}_{\pi}\left[\sum_{h=1}^{H}\widetilde{r}_{h}\mid \widetilde{P}\right].$$

\noindent\textbf{Best responses, Nash equilibrium.} For any policy $\mu$ of the max-player, there exists a best response policy $\nu^\dagger(\mu)$ of the min-player such that $V_h^{\mu,\nu^\dagger(\mu)}(s)=\inf_\nu V^{\mu,\nu}_h(s)$ for all $(s,h)\in\mathcal{S}\times[H]$. For simplicity, we denote $V^{\mu,\dagger}_h:=V_h^{\mu,\nu^\dagger(\mu)}$. Also, $\mu^\dagger(\nu)$ and $V_h^{\dagger,\nu}$ can be defined similarly. It is shown \citep{filar2012competitive} that there exists a pair of policies $(\mu^\star,\nu^\star)$ that are best responses against each other, \emph{i.e.} for all $(s,h)\in\mathcal{S}\times[H]$,
$$V_h^{\mu^\star,\dagger}(s)=V_h^{\mu^\star,\nu^\star}(s)=V_h^{\dagger,\nu^\star}(s).$$
We call the pair of policies Nash equilibrium of the Markov game, which further satisfies the following minimax property: for all $(s,h)\in\mathcal{S}\times[H]$,
$$\sup_\mu\inf_\nu V_h^{\mu,\nu}(s)=V_h^{\mu^\star,\nu^\star}(s)=\inf_\nu\sup_\mu V_h^{\mu,\nu}(s).$$
The value functions of $(\mu^\star,\nu^\star)$ are called Nash value functions and we denote $V^\star_h=V_h^{\mu^\star,\nu^\star},Q^\star_h=Q_h^{\mu^\star,\nu^\star}$for simplicity. Intuitively speaking, Nash equilibrium means that no player could benefit from switching her own policy.

\noindent\textbf{Non-Markov policies.} In this work, we consider general, history-dependent policies that may not be Markov policies. A general policy $\mu$ of the max-player is a set of mappings $\mu=\{\mu_h:\Omega\times(\mathcal{S}\times\mathcal{A}\times\mathcal{B}\times\mathbb{R})^{h-1}\times\mathcal{S}\rightarrow \Delta(\mathcal{A})\}$. The choice of the action at time step $h\in[H]$ depends on the history and a random sample $w\in\Omega$ that is shared among all time steps. A special case of general policies is a mixture of Markov policies, which will be used in this work. General policies for the min-player, the best response of a general policy can be defined similarly as Markov policies. Note that the best response to a non-Markov policy may be non-Markov. 

\noindent\textbf{Learning objective: regret.} Suppose $K$ is the number of episodes the agent plan to play and $\mu^k$ is the policy executed by the max-player in the $k$-th episode. Then the regret of the max-player is defined as 
$$\text{Regret}(K):=\sum_{k=1}^K [V_1^\star(s_1)-V_1^{\mu^k,\dagger}(s_1)].$$
Same as \citet{jin2022power}, our goal is to minimize the regret of the max-player, which focuses on learning the Nash policy of the max-player. By symmetry, the techniques readily extend to learning the Nash policy of the min-player.

\noindent\textbf{Batched reinforcement learning.} We measure the adaptivity of an online RL algorithm by batch complexity.
\begin{definition}
We say an algorithm has batch complexity $M$, if the algorithm pre-determines a group of lengths $\{T_i\}_{i\in[M]}$ such that $\sum_{i=1}^M T_i=K$. At the beginning of the $i$-th batch, the agent determines a general policy $\pi^i=(\mu^i,\nu^i)$ and follows $\pi^i$ for $T_i$ episodes. 
\end{definition}

Our goal is to minimize the batch complexity of our algorithm while achieving near-optimal regret.

\begin{remark}
    We highlight that our setting strictly generalizes the batched (single-agent) RL setting in \citet{qiao2022sample,zhang2022near}. This is because when the min-player plays a fixed and known policy, the definitions of MG, regret and batch complexity will reduce to the single-agent RL setting. Therefore, our setting is more complex and technically demanding by incorporating the min-player.
\end{remark}
	

\section{Main algorithms}\label{sec:alg}


In order to attain a batch complexity that is near-optimal, we extend the batch schedule derived from the arm-elimination algorithm for bandits \citep{cesa2013online} to an elimination  based algorithm for RL. The core concept behind our policy elimination approach revolves around maintaining a \emph{version space} $\Pi_A$ that encompasses the remaining policies for the max-player. In each batch, we enhance the estimated values of all policies within $\Pi_A$ and utilize these values to eliminate policies that cannot possibly be Nash policies. The overarching goal is that as the algorithm progresses, the policies remaining after elimination are already in approximate alignment with Nash equilibrium.

\begin{algorithm}[tbh]
	\caption{Main algorithm}\label{alg:main}
	\begin{algorithmic}[1]
		\STATE \textbf{Require}: Number of episodes $K$. The known reward $r$. Universal constants $C,N$. Failure probability $\delta$.
		\STATE \textbf{Initialize}: $T^{(k)}=K^{1-\frac{1}{2^{k}}}$, $k\leq K_{0}=O(\log\log K)$, $\Pi_A^{1}:=\{ \text{All Markov policies for the max-player}\}$, $\Pi_B:=\{ \text{All Markov policies for the min-player}\}$, $\iota=\log(2HSABK/\delta)$. 
		\FOR{$k=1,2,\cdots,K_{0}$}  
		\STATE \maroon{$\diamond$ Number of episodes in $k$-th stage:}
		\IF{$(N+1)T^{(1)}+(1+(k-1)N)T^{(2)}+\sum_{i=1}^{k}T^{(i)}\geq K$} 
		\STATE $T^{(k)}=K-(N+1)T^{(1)}-(1+(k-1)N)T^{(2)}-\sum_{i=1}^{k-1}T^{(i)}$ (o.w. $T^{(k)}=K^{1-\frac{1}{2^{k}}}$).
		\ENDIF
		\ENDFOR
		\FOR{$k=1,2$}
		\STATE \maroon{$\diamond$ Update the infrequent set $\mathcal{F}$ and construct an empirical estimate $\widehat{P}$ of the absorbing MG:}
		\STATE $\mathcal{F}^{k}$, $P^{int,k}$, $\pi_k$ = Crude Exploration$(\Pi_A^k,\Pi_B,T^{(k)}).$
		\STATE $\widehat{P}^{k}$ = Fine Exploration$(\mathcal{F}^{k},P^{int,k},\Pi_A^k,\Pi_B,T^{(k)},\pi_k,NT^{(k)}).$
            \STATE \maroon{$\diamond$ Policy elimination for the max-player:}
		\STATE $\Pi_A^{k+1}\leftarrow\{\mu\in\Pi_A^k\mid \inf_{\nu\in\Pi_B} V^{\mu,\nu}(r,\widehat{P}^{k})\geq \sup_{\mu\in\Pi_A^{k}}\inf_{\nu\in\Pi_B} V^{\mu,\nu}(r,\widehat{P}^{k})-2C(\sqrt{\frac{H^{3}S^{2}AB\iota}{T^{(k)}}}+\frac{H^{5}S^{3}A^{2}B^2\iota}{T^{(k)}})\}$.
		\ENDFOR
		\FOR{$k=3,4,\cdots,K_{0}$}
		\STATE \maroon{$\diamond$ Only update the empirical transition kernel:}
		\STATE $\widehat{P}^{k}$ = Fine Exploration$(\mathcal{F}^{2},P^{int,2},\Pi_A^k,\Pi_B,T^{(k)},\pi_2,NT^{(2)}).$
		\STATE $\Pi_A^{k+1}\leftarrow\{\mu\in\Pi_A^k\mid \inf_{\nu\in\Pi_B} V^{\mu,\nu}(r,\widehat{P}^{k})\geq \sup_{\mu\in\Pi_A^{k}}\inf_{\nu\in\Pi_B} V^{\mu,\nu}(r,\widehat{P}^{k})-2C(\sqrt{\frac{H^{3}S^{2}AB\iota}{T^{(k)}}}+\frac{H^{5}S^{3}A^{2}B^2\iota}{T^{(2)}})\}$.
		\ENDFOR
	\end{algorithmic}
\end{algorithm}

\begin{algorithm}[tbh]
	\caption{Crude Exploration}\label{alg:crude}
	\begin{algorithmic}[1]
		\STATE \textbf{Input}: Policy sets $\Pi_A$ (max) and $\Pi_B$ (min). Number of episodes $T$. Universal constant $C_1$.
		\STATE \textbf{Initialize}: $T_{0}=\frac{T}{H}$, $\mathcal{F}=\emptyset$, $\mathcal{D}=\emptyset$, $\iota=\log(2HSABK/\delta)$. $1_{h,s,a,b}$ is a reward function $r^{\prime}$ where $r^{\prime}_{h^{\prime}}(s^{\prime},a^{\prime},b^\prime)=\mathds{1}[(h^{\prime},s^{\prime},a^{\prime},b^\prime)=(h,s,a,b)]$. $s^{\dagger}$ is an additional absorbing state. $P^{int}$ is a transition kernel over the extended space $\mathcal{S}\cup\{s^\dagger\}\times\mathcal{A}\times\mathcal{B}$, initialized \emph{arbitrarily}.
		\STATE \textbf{Output}: Infrequent tuples $\mathcal{F}$. Intermediate transition kernel $P^{int}$. A uniformly explorative policy $\pi$.
		\FOR{$h=1,2,\cdots,H$}
		\STATE \maroon{$\diamond$ Construct and run policies to visit each state-action:}
		\FOR{$(s,a,b)\in \mathcal{S}\times \mathcal{A}\times\mathcal{B}$}
		\STATE $\pi_{h,s,a,b}=\mathrm{argmax}_{(\mu,\nu)\in \Pi_A\times\Pi_B}V^{\mu,\nu}(1_{h,s,a,b},P^{int}).$
            \ENDFOR
		\STATE Run $\pi_{h}=$ uniform mixture of $\{\pi_{h,s,a,b}\}_{(s,a,b)}$ for $T_{0}$ episodes, and add the trajectories into data set $\mathcal{D}$. 
		\FOR{$(s,a,b,s')\in\mathcal{S}\times\mathcal{A}\times\mathcal{B}\times\mathcal{S}$}
		\STATE $N_{h}(s,a,b,s^{\prime})=\text{count of}\ (h,s,a,b,s^{\prime})$ in $\mathcal{D}$.
		\ENDFOR
		\STATE \maroon{$\diamond$ Use $\mathcal{F}$ to store the infrequent tuples:}
		\STATE $\mathcal{F}=\mathcal{F}\cup \{(h,s,a,b,s^{\prime})|N_{h}(s,a,b,s^{\prime})\leq C_{1}H^{2}\iota\}.$
        \STATE \maroon{$\diamond$ Update the intermediate transition kernel using Algorithm~\ref{algo_transition_kernel}:}
		\STATE $P^{int} = \text{EstimateTransition}(\mathcal{D}, \mathcal{F}, s^{\dagger},h,P^{int})$
 		\STATE Reset data set $\mathcal{D}=\emptyset$.
		\ENDFOR
            \STATE \maroon{$\diamond$ Construct a uniformly explorative policy:}
            \STATE Policy $\pi\leftarrow$ uniform mixture of $\{\pi_{h,s,a,b}\}_{(h,s,a,b)}$.
		\STATE \textbf{Return}: $\{\mathcal{F},P^{int},\pi\}$.
	\end{algorithmic}
\end{algorithm}

\begin{algorithm}[tbh]
	\caption{Fine Exploration}\label{alg:fine}
	\begin{algorithmic}[1]
		\STATE \textbf{Input}: Infrequent tuples $\mathcal{F}$. Intermediate transition kernel $P^{int}$. Policy sets $\Pi_A$ and $\Pi_B$. Number of episodes $T$. Auxiliary policy $\pi^\prime$ with number of episodes $T^\prime$.
		\STATE \textbf{Initialize}: $\mathcal{D}=\emptyset$. $\widehat{P}=P^{int}$. $1_{h,s,a,b}$ is a reward function $r^{\prime}$ where $r^{\prime}_{h^{\prime}}(s^{\prime},a^{\prime},b^\prime)=\mathds{1}[(h^{\prime},s^{\prime},a^{\prime},b^\prime)=(h,s,a,b)]$.
		\STATE \textbf{Output}: Empirical estimate $\widehat{P}$.
		\STATE Construct explorative policy $\pi$ according to \eqref{equ:policy}.
		\STATE Run $\pi$ for $T$ episodes and run $\pi^\prime$ for $T^\prime$ episodes, add all the trajectories into data set $\mathcal{D}$.

		\STATE \maroon{$\diamond$ Construct an empirical estimate for $\widetilde{P}$ using Algorithm~\ref{algo_transition_kernel}:}
		\FOR{$h\in [H]$}
		\STATE $\widehat{P} = \mathrm{EstimateTransition}(\mathcal{D}, \mathcal{F}, s^{\dagger},h,\widehat{P}).$
		\ENDFOR
		\STATE \textbf{Return} $\widehat{P}$.
	\end{algorithmic}
\end{algorithm}


The primary hurdle lies in efficiently estimating the value function for all policy pairs with minimal sample usage. The challenge of uniform convergence typically involves estimating transition kernels, compounded by the necessity to address an exploration problem to visit specific state-action pairs at least once. Additionally, certain states may not be frequently visited by any policy pair. To tackle these challenges, we construct a surrogate Markov Game termed an ``absorbing MG'' featuring an absorbing state. This absorbing MG replaces problematic states with an absorbing state, denoted as $s^\dagger$. Consequently, all remaining states can be adequately visited by some policy in $\Pi_A\times\Pi_B$. Furthermore, its value function uniformly approximates the original MG for all relevant policies, thereby simplifying the problem to estimating the transition kernel of the absorbing MG.

\noindent\textbf{Main algorithm.} 
Given a budget of $K$ episodes, Algorithm \ref{alg:main} divides it into multiple stages with increasing length $T^{(k)}:= K^{1-1/2^k}$ for $k=1,2,3,...,K_0$. It is known \citep{cesa2013online} that the total number of stages $K_{0}$ is bounded by $O(\log\log K)$.
 
In Algorithm \ref{alg:main}, the first two stages involve three steps while all the following stages only contain the last two steps. Before we introduce the steps, we remark that such design is due to technical reasons. First, two implementations of Crude exploration could give an ``absorbing'' MG model that is sufficient for our purpose. In addition, to achieve the optimal batch complexity, the number of Crude exploraion's can only be a constant. Therefore, we give up the cleaner schedule where each stage contains all the three steps and choose the current schedule. Below is the three steps:

\begin{description}
\itemsep0em
\item[Step 1. Crude exploration] Explore each $(h,s,a,b)$ pair layer-by-layer using policies from the current version-space $\Pi_A^k$ (and $\Pi_B$). Establish an absorbing MG $\widetilde{P}$ and a \emph{crude} intermediate estimate ($P^{int}$) of $\widetilde{P}$.  
\item[Step 2. Fine exploration] Solve for a uniformly explorative policy to explore all tuples based on the crude estimate ($P^{int}$) of the absorbing MG. Construct a more accurate estimate ($\widehat{P}$) of the absorbing MG $\widetilde{P}$.  
\item[Step 3. Policy elimination] Evaluate all policies in $\Pi_A$ using $\inf_{\nu\in\Pi_B} V^{\mu,\nu}(r,\widehat{P})$ as an estimate of $V^{\mu,\dagger}(r,P)$. Update the version space $\Pi_A$ by eliminating all policies whose empirical value $\inf_{\nu\in\Pi_B} V^{\mu,\nu}(r,\widehat{P})$ is sub-optimal by a gap depending on the batch size. 
\end{description}


As the algorithm advances, assuming the high probability event that our confidence bounds hold, the Nash policy of the max-player is unlikely to be eliminated. At each stage, the remaining policies consistently outperform the LCB of the Nash policy, which will converge to the Nash value.

Next, we explain the first two steps of Algorithm \ref{alg:main}: Crude Exploration (Alg. \ref{alg:crude}) and Fine Exploration (Alg. \ref{alg:fine}).

\noindent\textbf{Layer-by-layer Exploration in Algorithm~\ref{alg:crude}.} Our Algorithm \ref{alg:crude} is a generalization of Algorithm 2 of \citet{qiao2022sample} to the multi-agent case. The motivation is that if we visit some tuple $(h,s,a,b,s^{\prime})$ for $O(H^{2}\iota)$ times, the empirical estimate of $P_{h}(s^{\prime}|s,a,b)$ would be multiplicatively accurate. Therefore, the challenge would be to visit each tuple as much as possible using policies from the input policy sets $\Pi_A$ and $\Pi_B$. However, some $(h,s,a,b,s^{\prime})$ tuples may be hard to reach by any policy in the policy set. To tackle this issue, we construct the set $\mathcal{F}$ to be all the tuples that do not have enough visitations. Consequently, for the tuples not in $\mathcal{F}$, the empirical estimate of transition is accurate enough, while for the tuples in $\mathcal{F}$, we will prove that no policy in the policy set could visit the tuple frequently thus they have little influence on the value function. 

More specifically, we explore the MG layer-by-layer. When exploring the $h$-th layer, we construct $\pi_{h,s,a,b}$'s to be the greedy policies under $P^{int}$, \emph{i.e.} $\pi_{h,s,a,b}$ could visit $(h,s,a,b)$ with the largest probability under $P^{int}$. Then we run a uniform mixture of $\{\pi_{h,s,a,b}\}_{(s,a,b)}$ (which is a general policy) for several episodes to collect the samples $\mathcal{D}$, after which we update the $h$-th layer of $P^{int}$ (i.e. $P^{int}_h$) using Algorithm~\ref{algo_transition_kernel} and $\mathcal{D}$. Finally, a uniformly explorative policy $\pi$ that can visit all tuples is stored for future use.

Based on infrequent tuples $\mathcal{F}$, the absorbing MG is established according to Definition~\ref{def2}. The construction is equivalent to first letting $\widetilde{P}=P$, and then moving the probability of $\widetilde{P}_{h}(s^{\prime}|s,a,b)$ to $\widetilde{P}_{h}(s^{\dagger}|s,a,b)$ for $(h,s,a,b,s^{\prime})\in\mathcal{F}$.
\begin{definition}[The absorbing MG $\widetilde{P}$]\label{def2}
    Given $\mathcal{F}$ and $P$, $\forall\, (h,s,a,b,s^{\prime}) \notin \mathcal{F}$, let
	$\widetilde{P}_{h}(s^{\prime}|s,a,b)=P_{h}(s^{\prime}|s,a,b).$
	For any $(h,s,a,b,s^{\prime})\in \mathcal{F}$, $\widetilde{P}_{h}(s^{\prime}|s,a,b)=0.$ For any $(h,s,a,b)\in[H]\times\mathcal{S}\times\mathcal{A}\times\mathcal{B}$, $\widetilde{P}_{h}(s^{\dagger}|s^{\dagger},a,b)=1$ and \[ \widetilde{P}_{h}(s^{\dagger}|s,a,b)=1-\sum_{s^{\prime}\in\mathcal{S}:(h,s,a,b,s^{\prime})\notin \mathcal{F}}\widetilde{P}_{h}(s^{\prime}|s,a,b).\]
\end{definition}

\begin{remark}
    Such absorbing structure has been applied under the single-agent case \citep{qiao2022sample,zhang2022near,zhang2023policy}, and here we generalize the definition to the multi-agent case.
    It will be shown that with high probability, for $(h,s,a,b,s^{\prime})$, either  $(1-\frac{1}{H})P^{int}_{h}(s^{\prime}|s,a,b)\leq \widetilde{P}_{h}(s^{\prime}|s,a,b) \leq (1+\frac{1}{H})P^{int}_{h}(s^{\prime}|s,a,b)$  or $P^{int}_{h}(s^{\prime}|s,a,b)=\widetilde{P}_{h}(s^{\prime}|s,a,b)=0$. Therefore, $\pi_{h,s,a,b}$'s are efficient in exploration. 
\end{remark}

Algorithm~\ref{algo_transition_kernel} and corresponding explanations are deferred to Appendix~\ref{appb1} due to space limit.
Here we remark that $P^{int}$ is the empirical estimate of $\widetilde{P}$ using the samples in $\mathcal{D}$. Besides, we discuss about the transition between the original MG and the absorbing MG in Appendix~\ref{appb2}.

\noindent\textbf{Fine exploration by Algorithm~\ref{alg:fine}.} The key behind Algorithm~\ref{alg:fine} is that with high probability, $V^{\mu,\nu}(1_{h,s,a,b},P^{int})$ is multiplicatively close to $V^{\mu,\nu}(1_{h,s,a,b},\widetilde{P})$ for all $(h,s,a,b)$ and remaining $(\mu,\nu)$. Therefore, $P^{int}$ can be used to guide the exploration. We construct the exploration policy $\pi$ as:

\begin{equation}\label{equ:policy}
\begin{split}
\pi&=\mathrm{argmin}_{\pi\in \Delta(\Pi_A\times\Pi_B)}\sup_{(\mu^\prime,\nu^\prime)\in\Pi_A\times\Pi_B} \\ &\sum_{h=1}^H\sum_{s,a,b}\frac{V^{\mu^\prime,\nu^\prime}(1_{h,s,a,b},P^{int})}{V^{\pi}(1_{h,s,a,b},P^{int})},
\end{split}
\end{equation}

which is a mixture of several policies from the current version space that could provide uniform coverage of all remaining policies. Then we run both $\pi$ (which is a general policy) and an auxiliary policy $\pi^\prime$\footnote{Running a uniformly explorative policy is for technical reason.} to collect samples. At last, $\widehat{P}$ is calculated as an empirical estimate of $\widetilde{P}$ by using Algorithm~\ref{algo_transition_kernel} and the data set $\mathcal{D}$.

Details about the implementation and computational efficiency of the algorithms are deferred to Section \ref{sec:results}.


\section{Main results}\label{sec:results}

In this section, we will state our main results, which formalize the algorithmic ideas explained in the previous section.

\begin{theorem}[Regret and batch complexity of Algorithm~\ref{alg:main}]\label{thm:main}
With probability $1-\delta$, Algorithm~\ref{alg:main} will have regret bounded by $\widetilde{O}(\sqrt{H^{2}S^{2}ABT})$, where $T:=KH$ is the number of steps. Furthermore, the batch complexity of Algorithm~\ref{alg:main} is bounded by $O(H+\log\log K)$.
\end{theorem}

Theorem \ref{thm:main} says that Algorithm \ref{alg:main} is able to achieve a regret bound with optimal dependence on $K$ while using only $O(\log \log K)$ batches. Due to space limit, the proof is sketched in Section \ref{sec:ps} with details in the Appendix. Now we discuss a few interesting aspects of the result.

\noindent\textbf{Near optimal batch complexity.} We state the following lower bound of batch complexity for all algorithms with $\widetilde{O}(\sqrt{K})$ regret bound, which implies that the batch complexity of our Algorithm \ref{alg:main} is nearly optimal.

\begin{theorem}[Lower bound]\label{thm:lower}
   For any algorithm with $O(\mathrm{poly}(S, A, B, H)\sqrt{K})$ regret bound, the batch complexity is at least $\Omega(H/\log_{A} K+ \log\log K)$.
\end{theorem}

Due to space limit, the proof is deferred to Appendix \ref{appc}.

\noindent\textbf{Transformation to PAC guarantee.} Various MARL applications require outputting a near optimal policy at the end of the algorithm. Below we provide a sample complexity upper bound for finding an $\epsilon$-\textbf{approximate Nash} policy for the max-player. We highlight that although we run general policies in the algorithm, the output policy can be a single Markov policy that is convenient to store and execute.

\begin{theorem}[Sample complexity]\label{thm:sample}
    For any $\epsilon>0$ and $\delta>0$, with probability $1-\delta$, Algorithm \ref{alg:main} could output a \textbf{Markov} policy $\mu$ of the max-player in $\widetilde{O}\left(\frac{H^3 S^2 AB}{\epsilon^2}\right)$ episodes such that $\mu$ is $\epsilon$-approximate Nash, \emph{i.e.}
    $$V^\star(r,P)-V^{\mu,\dagger}(r,P)\leq \epsilon.$$
\end{theorem}

The proof is deferred to Appendix \ref{appc}. By symmetry, an $\epsilon$-approximate Markov Nash policy for the min-player can be found using the same number of episodes.

\noindent\textbf{Dependence on $H,S,A,B$ in the regret.} Our regret bound is optimal in $T$. However, there is a gap of $\sqrt{S\min\{A,B\}}$ when compared to the information-theoretic limit of $\Omega(\sqrt{H^2S(A+B)T})$ that covers all algorithms (including those without adaptivity constraints) \citep{bai2020provable}. When applying our Algorithm \ref{alg:main} to the single-agent MDP setting where the min-player plays a fixed and known policy ($B=1$), the regret bound and batch complexity will be $\widetilde{O}(\sqrt{H^{2}S^{2}AT})$ and $O(H+\log\log K)$, respectively. The batch complexity is known to be optimal \citep{zhang2022near} while the regret is suboptimal by $\sqrt{S}$ when compared to the lower bound of $\Omega(\sqrt{H^2SAT})$ \citep{jin2018q}. We believe our analysis is tight and further improvements on $S$ will require new algorithmic ideas to handle the larger policy space compared to the single agent case.  It is an intriguing open problem whether we can design an algorithm with both optimal batch complexity and optimal regret bound.

\noindent\textbf{Computational efficiency.} Our Main Algorithm (Algorithm \ref{alg:main}) is not computationally efficient in general, since the algorithm needs to explicitly go over the remaining policy set to implement policy elimination, construct policy $\pi_{h,s,a,b}$'s in Crude exploration and $\pi$ in Fine exploration. These steps can be solved approximately in exponential time by enumerating over a tight covering set of Markov policies. For efficient surrogate of Algorithm \ref{alg:main}, we remark that a possible method is to apply softmax (or other differentiable) representation of the policy space and use gradient-based optimization techniques to find approximate solutions.
	

\section{Some discussions}\label{sec:dis}
\subsection{Application to bandit games ($H=S=1$)}
If for a two-player zero-sum Markov game, there is no multiple time steps and no states ($H=S=1$), then the Markov game reduces to a two-player zero-sum bandit game, which has wide real-world applications. According to Theorem \ref{thm:main}, a direct application of Algorithm \ref{alg:main} under the bandit setting will lead to a regret bound of $\widetilde{O}(\sqrt{ABK})$ and a batch complexity bound of $O(\log\log K)$. Moreover, with some mild revision to Algorithm \ref{alg:main}, the implementation can be computationally efficient. The result is summarized in Theorem \ref{thm:bandit} below. Note that the definition of batch complexity and regret is identical to the Markov game case.

\begin{theorem}\label{thm:bandit}
If we run Algorithm \ref{alg:bandit} (adapted from Algorithm \ref{alg:main}) under a two-player zero-sum bandit game for $K$ episodes, the batch complexity is bounded by $O(\log\log K)$ while the regret is $\widetilde{O}(\sqrt{ABK})$ with high probability. Furthermore, the algorithm is computationally efficient.
\end{theorem}

Due to space limit, Algorithm \ref{alg:bandit} and proof of Theorem \ref{thm:bandit} are deferred to Appendix \ref{sec:bandit}. To the best of our knowledge, Theorem \ref{thm:bandit} is the first result for learning  multi-player bandit games with low adaptivity, and the batch complexity is optimal while the regret has optimal dependence on $K$. The special case of bandit games where $B=1$ is the well-studied multi-armed bandits setting \citep{auer2002finite}. Previous works \citep{perchet2016batched,gao2019batched} designed algorithms for the MAB setting which achieve $O(\log\log K)$ batch complexity and $\widetilde{O}(\sqrt{AK})$ regret simultaneously, where both bounds are shown to be minimax optimal. In comparison, the above upper bounds can be achieved by directly plugging in $B=1$ to Theorem \ref{thm:bandit}. Therefore, our result can be considered as a provably efficient generalization of \citet{gao2019batched} to the game setting.

\subsection{Extension to the reward-free case}
In this part, we further consider the setting of low adaptive reward-free exploration \citep{jin2020reward}, where the goal is to output a near Nash policy for any possible reward function. Specifically, due to its nature that Crude Exploration (Algorithm \ref{alg:crude}) and Fine Exploration (Algorithm \ref{alg:fine}) do not use any information about the reward function $r$, these two algorithms can be leveraged in reward-free setting. Algorithm \ref{alg:rfe} uses Crude Exploration to construct the infrequent tuples $\mathcal{F}$ and the intermediate MG $P^{int}$. Then the algorithm uses Fine Exploration to get an empirical estimate $\widehat{P}$ of the absorbing MG $\widetilde{P}$. At last, for any reward function $r$, the algorithm outputs the Nash policy under the empirical MG.

\begin{algorithm}
	\caption{Algorithm for reward-free case}\label{alg:rfe}
	\begin{algorithmic}[1]
		\STATE \textbf{Input}: Episodes for crude exploration $N_{0}$, episodes for fine exploration $N_1$. Universal constant $N$. Failure probability $\delta$.
		\STATE \textbf{Initialize}: $\iota=\log(2HSAB(N_{0}+N_1)/\delta)$, $\Pi_A:=\{ \text{All Markov policies for the max-player}\}$, $\Pi_B:=\{ \text{All Markov policies for the min-player}\}$.
		\STATE \textbf{Output}: $\pi^{r}=(\mu^r,\nu^r)$ for any reward function $r$.
		\STATE $\mathcal{F}$, $P^{int}$, $\overline{\pi}$ = $\text{Crude Exploration}(\Pi_A,\Pi_B,N_{0})$.
		\STATE $\widehat{P}$ = $\text{Fine Exploration}(\mathcal{F},P^{int},\Pi_A,\Pi_B,N_1,\overline{\pi},NN_0).$
	\STATE $\mu^r=\mathrm{argmax}_{\mu\in\Pi_A}\inf_{\nu\in\Pi_B}V^{\mu,\nu}(r,\widehat{P})$ for any $r$.	
        \STATE $\nu^{r}=\mathrm{argmin}_{\nu\in\Pi_B}\sup_{\mu\in\Pi_A}V^{\mu,\nu}(r,\widehat{P})$ for any $r$.
		\STATE \textbf{Return} $\{\pi^{r}=(\mu^r,\nu^r)\}_{r}$.
	\end{algorithmic}
\end{algorithm}

The batch complexity and sample complexity of Algorithm \ref{alg:rfe} is summarized in Theorem \ref{thm:rfe} below (whose proof is deferred to Appendix \ref{sec:rfe} due to space limit).

\begin{theorem}\label{thm:rfe}
The batch complexity of Algorithm \ref{alg:rfe} is bounded by $H+2$. There exists a constant $c>0$ such that, for any $\epsilon >0$ and any $\delta>0$, if the number of total episodes $K$ satisfies that $$K>c\left(\frac{H^3 S^{2}AB \iota^{\prime}}{\epsilon^{2}}+\frac{H^{5}S^3 A^2 B^2 \iota^{\prime}}{\epsilon}\right),$$ where $\iota^{\prime}=\log(\frac{HSAB}{\epsilon\delta})$, then there exists a choice of $N_{0}$ and $N_1$ such that $N_{0}+N_1=K$ and with probability $1-\delta$, for any reward function $r$, Algorithm~\ref{alg:rfe} will output a policy pair $\pi^{r}=(\mu^r,\nu^r)$ that is $\epsilon$-approximate Nash.
\end{theorem}

The batch complexity of Algorithm \ref{alg:rfe} is known to be optimal up to logarithmic factors \citep{huang2021towards} while the sample complexity is optimal in $\epsilon$ \citep{jin2020reward}. Moreover, when applied to the special case of single-agent MDP, our sample complexity is $\widetilde{O}(\frac{H^3 S^2 A}{\epsilon^2})$, which matches the best known results even without adaptivity constraints \citep{menard2021fast,zhang2020nearly}. We highlight that our sample complexity improves significantly over previous algorithms with near optimal adaptivity \citep{qiao2022sample,qiao2022near}, whose sample complexities are both $\widetilde{O}(\frac{H^5 S^2 A}{\epsilon^2})$.
Last but not least, compared to Algorithm \ref{alg:main}, our Algorithm \ref{alg:rfe} could output near Nash policies for both players simultaneously by running a single algorithm.
	

\section{Proof overview}\label{sec:ps}

In this section, we summarize the proof of Theorem~\ref{thm:main}. The analysis contains two main parts: the batch complexity bound and the regret bound. The batch complexity bound can be directly derived from the schedule of Algorithm~\ref{alg:main}.

\noindent\textbf{Upper bound for batch complexity.} Crude exploration (Algorithm \ref{alg:crude}) can be run in $H$ batches while Fine exploration (Algorithm \ref{alg:fine}) can be run in $2$ batches. Since Algorithm \ref{alg:main} contains 2 Crude exploration's and $O(\log\log K)$ Fine exploration's, we have the conclusion that the batch complexity of Algorithm~\ref{alg:main} is bounded by $O(H+\log\log K)$.

However, such an elimination schedule requires the algorithm to run the same exploration policy for a large number of episodes before being able to switch to another policy, which is the main technical hurdle to the regret analysis.

\noindent\textbf{Regret upper bound.} The core of the regret analysis is to construct a uniform off-policy evaluation bound that covers all remaining policies.
The remaining policy set (for the max-player) at the beginning of stage $k$ is $\Pi_A^{k}$. Assume with high probability, for all $(\mu,\nu)\in\Pi_A^{k}\times\Pi_B$, we can estimate $V^{\mu,\nu}(r,P)$  to $\epsilon_{k}$ accuracy using $\widehat{P}^k$, then the difference between $\inf_{\nu\in\Pi_B}V^{\mu,\nu}(r,\widehat{P}^k)$ and $V^{\mu,\dagger}(r,P)$ is bounded by $\epsilon_k$ uniformly for all $\mu\in\Pi_A^k$. Therefore if we eliminate all policies $\mu$ that are at least $2\epsilon_{k}$ sub-optimal in the sense of $\inf_{\nu\in\Pi_B}V^{\mu,\nu}(r,\widehat{P}^k)$, the Nash policy will not be eliminated and all the policies remaining will be $4\epsilon_{k}$-approximate Nash. Summing up the regret of all stages, we have with high probability, the total regret is bounded by
\begin{equation}\label{equr}
\text{Regret}(K)\leq O(HT^{(1)})+O\left(\sum_{k=2}^{K_{0}}T^{(k)}\times \epsilon_{k-1}\right).
\end{equation}
The following key lemma provides an upper bound of $\epsilon_{k-1}$ given that we use the empirical transition kernel $\widehat{P}^k$ in Algorithm \ref{alg:main} to estimate $V^{\mu,\nu}(r,P)$. 
\begin{lemma}\label{lem16}
	W.h.p, for any $k$ and $(\mu,\nu)\in\Pi_A^{k}\times\Pi_B$,
	$$\left|V^{\mu,\nu}(r,\widehat{P}^{k})-V^{\mu,\nu}(r,P)\right|\leq \widetilde{O}\left(\sqrt{\frac{H^{3}S^{2}AB}{T^{(k)}}}\right).$$
\end{lemma}
The proof of Lemma~\ref{lem16} requires controlling both the ``bias'' and ``variance'' part of estimation error. The ``bias'' refers to the difference between the true MG and the absorbing MG, while the ``variance'' refers to the statistical error in estimating the value functions of the absorbing MG using 
$\widehat{P}^{k}$. For simplicity, in the following discussion, we omit the stage number $k$ and the discussion holds true for all $k$.

\noindent\textbf{The ``bias'': difference between $P$ and $\widetilde{P}$.} First, it holds that if the visitation number of a tuple $(h,s,a,b,s^{\prime})$ is larger than $O(H^{2}\iota)$, with high probability, 
{\small
\begin{align*}
(1-\frac{1}{H})P^{int}_{h}(s^{\prime}|s,a,b)\leq \widetilde{P}_{h}(s^{\prime}|s,a,b) \leq (1+\frac{1}{H})P^{int}_{h}(s^{\prime}|s,a,b).
\end{align*}
}
According to the definition of $\mathcal{F}$ in Algorithm~\ref{alg:crude}, we have the above inequality is true for any $(h,s,a,b,s^{\prime})$. Then we have for any $(h,s,a,b)$ and any policy $(\mu,\nu)\in\Pi_A\times\Pi_B$, 
\begin{align*}
\frac{1}{4}V^{\mu,\nu}(1_{h,s,a,b},P^{int}) \leq V^{\mu,\nu}(1_{h,s,a,b},\widetilde{P}) \leq 3V^{\mu,\nu}(1_{h,s,a,b},P^{int}).
\end{align*}
Due to the construction of $\pi_{h,s,a,b}$, we can see that it is efficient in visiting the tuple $(h,s,a,b)$ under the true MG. Therefore, we are able to bound the difference between $P$ and $\widetilde{P}$ in the sense of value function. 

\begin{lemma}
W.h.p, for any policy $(\mu,\nu)\in\Pi_A\times\Pi_B$, 
$$\left|V^{\mu,\nu}(r,\widetilde{P})-V^{\mu,\nu}(r,P)\right|\leq \widetilde{O}\left(\frac{H^{5}S^{3}A^{2}B^2}{T}\right).$$ 
\end{lemma}

\noindent\textbf{The ``variance'': difference between $\widehat{P}$ and $\widetilde{P}$.} Since $P^{int}$ is close to $\widetilde{P}$, we can prove the following lemma.
\begin{lemma}
W.h.p, for any policy $(\mu,\nu)\in\Pi_A\times\Pi_B$,
$$\left|V^{\mu,\nu}(r,\widehat{P})-V^{\mu,\nu}(r,\widetilde{P})\right|\leq\widetilde{O}\left(\sqrt{\frac{H^{3}S^{2}AB}{T}}\right).$$
\end{lemma}

\noindent\textbf{Put everything together.} Combining the bounds of the ``bias'' term and the ``variance'' term, because of triangular inequality, we have the conclusion in Lemma~\ref{lem16} holds (the ``bias'' term is lower order). Then the proof of regret bound is completed by plugging in $\epsilon_{k}=\widetilde{O}\left(\sqrt{\frac{H^{3}S^{2}AB}{T^{(k)}}}\right)$ in \eqref{equr}.
	

\section{Conclusion}\label{sec:conclude}
In this paper, we take the initial step to study the well-motivated problem of low adaptive multi-agent reinforcement learning. We design Algorithm \ref{alg:main} that achieves the optimal batch complexity of $O(H+\log\log K)$ and near optimal regret of $\widetilde{O}(\sqrt{H^2 S^2 ABT})$. Furthermore, our techniques naturally extend to bandit games and reward-free exploration with low adaptivity, where our results generalize or improve various previous works in the single-agent case. Future extensions are numerous, it remains open to address the computational efficiency, study Markov Games with general function approximation, as well as make the algorithms practical. We leave those as future works.

	\section*{Acknowledgments}
	The research is partially supported by NSF Awards \#2007117 and \#2003257. 
	
	\bibliographystyle{plainnat}
	\bibliography{example_paper}
	
	\newpage
	\appendix

\section{Extended related work}\label{appr}

\textbf{Low regret reinforcement learning algorithms.}
There has been a long line of works \citep{azar2017minimax,jin2018q,dann2019policy,zhang2020almost,zanette2019tighter,xu2023doubly} focusing on regret minimization for online reinforcement learning. \citet{azar2017minimax} first achieved the optimal regret bound for stationary tabular MDP using model-based approach. \citet{dann2019policy} designed an algorithm to get optimal regret bound and policy certificates at the same time. \citet{jin2018q} proved a regret bound for Q-learning while the regret bound was improved by \citet{zhang2020almost} via incorporating advantange decomposition. \citet{zanette2019tighter} designed the algorithm EULER to get an instance dependent regret bound, which also matches the lower bound in the worst case. There are also works \citep{qiao2022offline,qiao2023near} designing RL algorithms with differential privacy guarantees \citep{dwork2006calibrating,zhao2022differentially}.

\textbf{Low switching algorithms for bandits and RL.} \citet{auer2002finite} first studied multi-armed bandits with low switching cost by proposing the famous UCB2 algorithm. The optimal switching cost bound is later achieved by \citet{cesa2013online}. For multi-armed bandits with $A$ arms and $T$ episodes, they designed an algorithm with the optimal $\widetilde{O}(\sqrt{AT})$ regret while the switching cost is only $O(A\log\log T)$. \citet{simchi2019phase} generalized the result by showing that a switching cost of order $A\log\log T$ is necessary for getting the optimal $\widetilde{O}(\sqrt{T})$ regret bound. For stochastic linear bandits, \citet{abbasi2011improved} achieved the optimal regret $\widetilde{O}(d\sqrt{T})$ with $O(d\log T)$ policy switches by applying doubling trick. Under a slightly different setting, \citet{ruan2021linear} improved the result by reducing the switching cost to $O(\log\log T)$ while keeping the regret bound. \citet{bai2019provably} first studied the problem under tabular MDP. They reached regret bound of $\widetilde{O}(\sqrt{H^3SAT})$ with local switching cost $O(H^3SA\log T)$ by applying doubling trick to Q-learning. Both regret and switching cost are improved by \citet{zhang2020almost} using advantage decomposition. \citet{qiao2022sample} established that a global switching cost of order $HSA\log\log T$ is enough and necessary to achieve the optimal $\widetilde{O}(\sqrt{T})$ regret. For linear MDP, \citet{gao2021provably} arrived at regret bound $\widetilde{O}(\sqrt{d^3H^3T})$ with global switching cost $O(dH\log T)$ by applying doubling trick to LSVI-UCB. \citet{wang2021provably} generalized the above result to work for arbitrary budget of switching cost. Later, the regret bound is improved to the minimax optimal rate $\widetilde{O}(\sqrt{d^2 H^2 T})$ by \citet{he2023nearly}, while the switching cost remains the same. Beyond the linear MDP model, \citet{qiao2023logarithmic} designed low switching algorithms for linear Bellman complete MDP and MDP with general linear approximation, while \citet{kong2021online,velegkas2022best,zhao2023nearly,xiong2023general} considered MDPs with general function approximation. All of these algorithms achieved a switching cost bound depending only logarithmically on $T$. \citet{shi2023near} considered low switching adversarial reinforcement learning. On the empirical side, \citet{xu2022beyond} constructed an empirical benchmark for low switching RL. In addition, low switching algorithms can be further applied to deal with RL with delayed feedback \citep{yang2023reduction} or to achieve short burn-in time under discounted MDPs \citep{ji2023regret}.

\textbf{Batched bandits and RL.} For multi-armed bandits with $A$ arms and $T$ episodes, \citet{cesa2013online} designed an algorithm with $\widetilde{O}(\sqrt{AT})$ regret while the batch complexity is only $O(\log\log T)$. \citet{perchet2016batched} proved that $\Omega(\log\log T)$ batches are necessary for a regret bound of $\widetilde{O}(\sqrt{T})$ under 2-armed bandits. The result is generalized to $K$-armed bandits by \citet{gao2019batched}. For stochastic linear bandits, \citet{han2020sequential} achieved a regret bound of $\widetilde{O}(\sqrt{T})$ while the batch complexity is only $O(\log\log T)$. The result is improved by \citet{ruan2021linear} via using weaker assumptions. For batched RL setting, \citet{qiao2022sample} showed that their algorithm uses the optimal $O(H+\log\log T)$ batches to achieve the optimal $\widetilde{O}(\sqrt{T})$ regret. Later, \citet{zhang2022near} incorporated the idea of optimal experimental design to get near optimal regret bound and computational efficiency. For linear MDP, \citet{wang2021provably} first designed an algorithm with low batch complexity. Later, \citet{johnson2023sample} proved a dimension-dependent lower bound of batch complexity for any sample-efficient algorithms. The deployment efficient algorithms for pure exploration by \citet{huang2021towards,qiao2022near} also satisfy the definition of batched RL. Beyond linear MDPs, \citet{modi2021model} designed deployment efficient algorithms for low rank MDPs while \citet{xiong2023general} proposed batched algorithms for MDPs with general function approximation.

\textbf{Reward-free exploration.} The problem of reward-free exploration is first studied by \citet{jin2020reward}. They used the EULER algorithm \citep{zanette2019tighter} to visit each state-action pair as much as possible, and their sample complexity is $\widetilde{O}(H^{5}S^{2}A/\epsilon^{2})$. \citet{kaufmann2021adaptive} designed an algorithm which requires $\widetilde{O}(H^4 S^{2}A/\epsilon^{2})$ episodes to output a near-optimal policy for any reward function. \citet{menard2021fast} further improved the sample complexity to $\widetilde{O}(H^3 S^{2}A/\epsilon^{2})$. \citet{zhang2020nearly} considered a more general horizon-free setting with stationary transition kernel. They constructed a novel condition to achieve the optimal sample complexity $\widetilde{O}(S^{2}A/\epsilon^{2})$ under their specific setting. Also, their result can be transferred to achieve $\widetilde{O}(H^3 S^{2}A/\epsilon^{2})$ sample complexity under the standard setting where $r_{h}\in [0,1]$ and the transition kernel is non-stationary. When considering low adaptivity, \citet{qiao2022sample} designed an algorithm with the optimal $O(HSA)$ switching cost and $\widetilde{O}(H^{5}S^{2}A/\epsilon^{2})$ sample complexity. There is a relevant setting named task-agnostic exploration. \citet{zhang2020task} designed an algorithm that could find $\epsilon$-optimal policies for $N$ arbitrary tasks within at most $\widetilde{O} (H^{5}SA\log N/\epsilon^{2})$ exploration episodes. Recently, \citet{li2023minimax} studied reward-agnostic exploration in reinforcement learning, which incorporates the two settings above. For MDP with linear function approximation, a series of papers \citep{wang2020reward,zanette2020provably,zhang2021reward,chen2021near,wagenmaker2022reward,huang2021towards,qiao2022near,hu2022towards} analyzed reward-free exploration.
Among the works, \citet{huang2021towards,qiao2022near} achieved near optimal adaptivity at the same time. Beyond linear MDPs, \citet{qiu2021reward,chen2022statistical} considered reward-free exploration in MDPs with general function approximation. Reward-free RL is also known to be helpful to constrained reinforcement learning \citep{miryoosefi2022simple}.

\section{Missing algorithm: EstimateTransition (Algorithm~\ref{algo_transition_kernel}) and some explanation}\label{appb1}

Our Algorithm \ref{algo_transition_kernel} is a generalization of Algorithm 5 in \citet{qiao2022sample} to the multi-agent case. The algorithm takes a data set $\mathcal{D}$, a set of infrequent tuples $\mathcal{F}$, a transition model $P$ and the target layer $h$ to update as inputs. The output is an updated transition model where the $h$-th layer is derived according to $\mathcal{D}$ while the remaining layers stay the same. The construction of $P_{h}$ is for the tuples in $\mathcal{F}$, $P_{h}(s^{\prime}|s,a,b)=0$. For the tuples not in $\mathcal{F}$, $P_{h}(s^{\prime}|s,a,b)$ is the empirical estimate from $\mathcal{D}$. At last, $P_{h}(s^{\dagger}|s,a,b)=1-\sum_{s^{\prime}\in\mathcal{S}:(h,s,a,b,s^{\prime})\notin \mathcal{F}}P_{h}(s^{\prime}|s,a,b)$ holds so that $P_{h}$ is a valid transition kernel. In other words, the construction is similar to the construction of $\widetilde{P}$. We first let $P$ be the empirical estimate based on $\mathcal{D}$, then for $(h,s,a,b,s^{\prime})\in\mathcal{F}$, we move the probability of $P_{h}(s^{\prime}|s,a,b)$ to $P_{h}(s^{\dagger}|s,a,b)$.

\begin{algorithm}
	\caption{Compute Transition Kernel (EstimateTransition)}\label{algo_transition_kernel}
	\begin{algorithmic}[1]
		\STATE \textbf{Require}: Data set $\mathcal{D}$, infrequent tuples $\mathcal{F}$, absorbing state $s^{\dagger}$, the target layer $h$, transition kernel $P$.
		\STATE \textbf{Output}: Estimated transition kernel $P$ from data set $\mathcal{D}$.
		\STATE \maroon{$\diamond$ Count the visitation number of each state-action pairs from the target layer $h$:}
		\FOR{$(s,a,b,s')\in\mathcal{S}\times\mathcal{A}\times\mathcal{B}\times\mathcal{S}$}
		\STATE $N_{h}(s,a,b,s^{\prime})=\text{count of}\ (h,s,a,b,s^{\prime})$ in $\mathcal{D}$.
		\STATE $N_{h}(s,a,b)=\text{count of}\ (h,s,a,b)$ in $\mathcal{D}$.
		\ENDFOR
        \STATE \maroon{$\diamond$ Update the $h$-th layer of the transition kernel:}
		\FOR{$(s,a,b,s^{\prime})\in \mathcal{S}\times \mathcal{A}\times\mathcal{B}\times\mathcal{S}$ s.t. $(h,s,a,b,s^{\prime})\in \mathcal{F}$}
		\STATE $P_{h}(s^{\prime}|s,a,b)=0.$
		\ENDFOR
		\FOR{$(s,a,b,s^{\prime})\in \mathcal{S}\times \mathcal{A}\times\mathcal{B}\times\mathcal{S}$ s.t. $(h,s,a,b,s^{\prime})\notin \mathcal{F}$}
		\STATE $P_{h}(s^{\prime}|s,a,b)=\frac{N_{h}(s,a,b,s^{\prime})}{N_{h}(s,a,b)}.$ 
		\ENDFOR
		\FOR{$(s,a,b)\in \mathcal{S}\times \mathcal{A}\times\mathcal{B}$}
		\STATE $P_{h}(s^{\dagger}|s,a,b)=1-\sum_{s^{\prime}\in\mathcal{S}:(h,s,a,b,s^{\prime})\notin \mathcal{F}}P_{h}(s^{\prime}|s,a,b).$
		\ENDFOR
		\FOR{$(a,b)\in \mathcal{A}\times\mathcal{B}$}
		\STATE $P_{h}(s^{\dagger}|s^{\dagger},a,b)=1.$
		\ENDFOR
		\STATE \textbf{Return} $P$.
	\end{algorithmic}
\end{algorithm}

\section{Transition between original MG and absorbing MG}\label{appb2}
For any reward function $r$ defined under the original MG $P$, we overload the notation and also use it under the absorbing MG. The extended definition of $r$ under the absorbing MG is shown below: \\
$$r(s,a,b) = \begin{cases} r(s,a,b), &   s\in \mathcal{S},  \\ 0,& s=s^{\dagger}. \end{cases}$$
For any policy $\pi= (\mu,\nu)$ defined under the original MG $P$, we overload the notation and also use it under the absorbing MG. The extended definition of $(\mu,\nu)$ under the absorbing MG is shown below: \\
$$\mu(\cdot|s) = \begin{cases} \mu(\cdot|s), &  s\in \mathcal{S}, \\ \text{arbitrary distribution,} & s=s^{\dagger}. \end{cases}$$
$$\nu(\cdot|s) = \begin{cases} \nu(\cdot|s), &  s\in \mathcal{S}, \\ \text{arbitrary distribution,} & s=s^{\dagger}. \end{cases}$$
With the definition of $r$ and $\pi$, the value function under the absorbing MG is independent of the \emph{arbitrary distribution} since there will be no more reward once the agent enters the absorbing state $s^{\dagger}$. The policies defined under the absorbing MG can be directly applied under the real MG because such policies have definition for any $s\in \mathcal{S}$.

In this paper, $P$ is the real MG, which is under original MG. In each stage, $\widetilde{P}$ is an absorbing MG constructed based on infrequent tuples $\mathcal{F}$ and the real MG $P$. When we run the algorithm, we do not have access to $\widetilde{P}$, but we know exactly the intermediate transition kernel $P^{int}$, which is also an absorbing MG. In Algorithm~\ref{alg:fine}, the $\widehat{P}$ we construct is the empirical estimate of $\widetilde{P}$, which is also an absorbing MG. In the proof of this paper, a large part of discussion is under the framework of absorbing MG. When we specify that the discussion is under absorbing MG with absorbing state $s^{\dagger}$, any transition kernel $P^{\prime}$ satisfies $P^{\prime}_{h}(s^{\dagger}|s^{\dagger},a,b)=1$ for any $(a,b,h)\in\mathcal{A}\times\mathcal{B}\times[H]$. For the reward functions in this paper, they are all defined under original MG, when applied under absorbing MG, the transition rule follows what we just discussed.

\section{Proof of lemmas regarding Crude Exploration  (Algorithm~\ref{alg:crude})}

We first prove an upper bound for batch complexity.
\begin{lemma}\label{lem1}
	The batch complexity of Algorithm~\ref{alg:crude} is bounded by $H$.
\end{lemma}

\begin{proof}[Proof of Lemma \ref{lem1}]
    Since each $\pi_{h}=$ uniform mixture of $\{\pi_{h,s,a,b}\}_{(s,a,b)}$ is one general policy and we run $H$ such $\pi_h$'s in total, the number of batches is bounded by $H$.
\end{proof}

Note that our Crude Exploration (Algorithm \ref{alg:crude}) is adapted from Algorithm 2 in \citet{qiao2022sample} for the single-agent RL case, and the difference is that here we replace the single version space $\phi$ in \citet{qiao2022sample} with a product policy space $\Pi_A\times\Pi_B$. Therefore, the results in \citet{qiao2022sample} directly extend to the multi-agent case. We begin with the property of $P^{int}$ and $\widetilde{P}$.

\begin{lemma}[Lemma E.3 in \citet{qiao2022sample}]\label{lem6}
With probability $1-\frac{S\delta}{K}$, $\forall\, (h,s,a,b,s^{\prime})\in[H]\times \mathcal{S}\times\mathcal{A}\times\mathcal{B}\times\mathcal{S}$ such that $(h,s,a,b,s^{\prime})\notin \mathcal{F}$, it holds that
	$$ (1-\frac{1}{H})P^{int}_{h}(s^{\prime}|s,a,b)\leq \widetilde{P}_{h}(s^{\prime}|s,a,b) \leq (1+\frac{1}{H})P^{int}_{h}(s^{\prime}|s,a,b). $$
 In addition, we have that $\forall\, (h,s,a,b,s^{\prime})\in \mathcal{F}$, $\widetilde{P}_{h}(s^{\prime}|s,a,b)=P^{int}_{h}(s^{\prime}|s,a,b)=0$.
\end{lemma}

From Lemma~\ref{lem6}, we can see that for those tuples $(h,s,a,b,s^{\prime})$ not in $\mathcal{F}$, the estimate of the transition kernel satisfies $ (1-\frac{1}{H})P^{int}_{h}(s^{\prime}|s,a,b)\leq \widetilde{P}_{h}(s^{\prime}|s,a,b) \leq (1+\frac{1}{H})P^{int}_{h}(s^{\prime}|s,a,b)$ with high probability. In addition, for those states $(h,s,a,b,s^{\prime})\in \mathcal{F}$, $P^{int}_{h}(s^{\prime}|s,a,b)=\widetilde{P}_{h}(s^{\prime}|s,a,b)=0$, which means this inequality holds for all $(h,s,a,b,s^{\prime})\in [H]\times\mathcal{S}\times\mathcal{A}\times\mathcal{B}\times\mathcal{S}$. For simplicity, we use a new definition $\theta$-multiplicatively accurate to describe the relationship between $P^{int}$ and $\widetilde{P}$.

\begin{definition}[$\theta$-multiplicatively accurate for transition kernels (under absorbing MGs)]\label{def3}
	Under the absorbing MG with absorbing state $s^{\dagger}$, a transition kernel $P^{\prime}$ is $\theta$-multiplicatively accurate to another transition kernel $P^{\prime\prime}$ if $$(1-\theta)P^{\prime}_{h}(s^{\prime}|s,a,b) \leq P^{\prime\prime}_{h}(s^{\prime}|s,a,b) \leq (1+\theta)P^{\prime}_{h}(s^{\prime}|s,a,b)$$ for all $(h,s,a,b,s^{\prime})\in[H]\times\mathcal{S}\times\mathcal{A}\times\mathcal{B}\times\mathcal{S}$ and there is no requirement for the case when $s^{\prime}=s^{\dagger}$.
\end{definition}

Because of Lemma~\ref{lem6}, we have that with probability $1-\frac{S\delta}{K}$, $P^{int}$ is $\frac{1}{H}$-multiplicatively accurate to $\widetilde{P}$. Next, we will compare the visitation probability of each state $(h,s,a,b)$ under two transition kernels that are close to each other. 

\begin{lemma}[Lemma E.5 in \citet{qiao2022sample}]\label{lem2}
	Define $1_{h,s,a,b}$ to be the reward function $r^{\prime}$ such that $r^{\prime}_{h^{\prime}}(s^{\prime},a^{\prime},b^\prime)=\mathds{1}[(h^{\prime},s^{\prime},a^{\prime},b^\prime)=(h,s,a,b)]$. Similarly, define  $1_{h,s}$ to be the reward function $r^{\prime}$ such that $r^{\prime}_{h^{\prime}}(s^{\prime},a^{\prime},b^\prime)=\mathds{1}[(h^{\prime},s^{\prime})=(h,s)]$. Then $V^{\mu,\nu}(1_{h,s,a,b},P^{\prime})$ and $V^{\mu,\nu}(1_{h,s},P^{\prime})$ denote the visitation probability of $(h,s,a,b)$ and $(h,s)$, respectively, under $(\mu,\nu)$ and $P^{\prime}$. Under the absorbing MG with absorbing state $s^{\dagger}$, if $P^{\prime}$ is $\frac{1}{H}$-multiplicatively accurate to $P^{\prime\prime}$, for any policy pair $(\mu,\nu)$ and any $(h,s,a,b)\in[H]\times\mathcal{S}\times\mathcal{A}\times\mathcal{B}$, it holds that
	$$\frac{1}{4}V^{\mu,\nu}(1_{h,s,a,b},P^{\prime}) \leq V^{\mu,\nu}(1_{h,s,a,b},P^{\prime\prime}) \leq 3V^{\mu,\nu}(1_{h,s,a,b},P^{\prime}).$$
\end{lemma}

Combining Lemma~\ref{lem6} and Lemma~\ref{lem2}, we have with high probability, for any policy pair $(\mu,\nu)$ and any $(h,s,a,b)\in [H]\times \mathcal{S}\times \mathcal{A}\times\mathcal{B}$, 
\begin{equation}\label{equ:close}
    \frac{1}{4}V^{\mu,\nu}(1_{h,s,a,b},P^{int}) \leq V^{\mu,\nu}(1_{h,s,a,b},\widetilde{P}) \leq 3V^{\mu,\nu}(1_{h,s,a,b},P^{int}).
\end{equation} 
The structure of the absorbing MG also gives rise to the following lemma
 about the relationship between $\widetilde{P}$ and $P$.
\begin{lemma}[Lemma E.6 in \citet{qiao2022sample}]\label{rem10}
 For any policy pair $(\mu,\nu)$ and any $(h,s,a,b)\in [H]\times \mathcal{S}\times \mathcal{A}\times\mathcal{B}$, $$V^{\mu,\nu}(1_{h,s,a,b},P)\geq V^{\mu,\nu}(1_{h,s,a,b},\widetilde{P}).$$
\end{lemma}

Now we are ready to state the key lemma about the difference between $\widetilde{P}$ and $P$. 

\begin{lemma}[Lemma E.12 in \citet{qiao2022sample}]\label{lemfirst}
	There exists a universal constant $c_1>0$, such that with probability $1-\frac{2S\delta}{K}$, it holds that for any policy pair $(\mu,\nu)\in\Pi_A\times\Pi_B$ and reward function $r^{\prime}$,
  $$0\leq V^{\mu,\nu}(r^{\prime},P)-V^{\mu,\nu}(r^{\prime},\widetilde{P})\leq \frac{c_1 H^5 S^{3}A^{2}B^2\iota}{T}.$$
\end{lemma}

Recall that at the end of Crude Exploration (Algorithm \ref{alg:crude}), we store a policy $\pi=$ uniform mixture of $\{\pi_{h,s,a,b}\}_{(h,s,a,b)}$ for future use. Below we prove some properties of the uniformly explorative policy $\pi$.

\begin{lemma}\label{lem:policy}
    There exists a universal constant $c_2$, with probability at least $1-\frac{S\delta}{K}$, the policy $\pi=$ uniform mixture of $\{\pi_{h,s,a,b}\}_{(h,s,a,b)}$ satisfies that for all $(h,s,a,b,s^\prime)\in[H]\times\mathcal{S}\times\mathcal{A}\times\mathcal{B}\times\mathcal{S}$ such that $(h,s,a,b,s^\prime)\notin \mathcal{F}$,
    $$T\cdot\mathbb{P}_{\pi}[(h,s,a,b,s^\prime)|P]\geq c_2H^2\iota,$$
    where $\mathbb{P}_{\pi}[(h,s,a,b,s^\prime)|P]$ means the probability of reaching $(s,a,b,s^\prime)$ at time step $h$ under policy $\pi$ and original MG $P$.
\end{lemma}

\begin{proof}[Proof of Lemma \ref{lem:policy}]
    Recall that we run $\pi_{h}=$ uniform mixture of $\{\pi_{h,s,a,b}\}_{(s,a,b)}$ for $T_{0}=\frac{T}{H}$ episodes in Algorithm \ref{alg:crude}. Denote $E_{h,s,a,b,s^\prime}=T\cdot\mathbb{P}_{\pi}[(h,s,a,b,s^\prime)|P]$ and $\overline{E}_{h,s,a,b,s^\prime}=\frac{T}{H}\cdot\mathbb{P}_{\pi_h}[(h,s,a,b,s^\prime)|P]$. Then it holds that $E_{h,s,a,b,s^\prime}\geq\overline{E}_{h,s,a,b,s^\prime}$, because  $\pi$ is equivalent to a uniform mixture of $\{\pi_h\}_{h\in[H]}$. Therefore, for any $(h,s,a,b,s^\prime)\notin \mathcal{F}$, with probability $1-\frac{\delta}{HSABK}$,
\begin{equation}\label{equ:expect}
\begin{split}
    &E_{h,s,a,b,s^\prime}+\sqrt{3E_{h,s,a,b,s^\prime}\iota}\geq \overline{E}_{h,s,a,b,s^\prime}+\sqrt{3\overline{E}_{h,s,a,b,s^\prime}\iota}\\\geq& \text{visitation count of }(h,s,a,b,s^\prime) \geq C_1 H^2\iota, 
\end{split}
\end{equation}   
 where the second inequality holds with probability $1-\frac{\delta}{HSABK}$ due to Multiplicative Chernoff bound (Lemma \ref{lem:chernoff}). The last inequality results from the definition of infrequent tuples $\mathcal{F}$.

 Finally, based on \eqref{equ:expect}, with probability $1-\frac{\delta}{HSABK}$, $E_{h,s,a,b,s^\prime}\geq c_2H^2\iota$. The conclusion holds because of a union bound over $(h,s,a,b,s^\prime)\notin \mathcal{F}$. 
\end{proof}

\section{Proof of lemmas regarding Fine Exploration  (Algorithm~\ref{alg:fine})}\label{appb}
We first state a conclusion about the batch complexity of Algorithm~\ref{alg:fine}.
\begin{lemma}\label{lem11}
	The batch complexity of Algorithm~\ref{alg:fine} is bounded by 2.
\end{lemma}

\begin{proof}[Proof of Lemma~\ref{lem11}]
   Algorithm~\ref{alg:fine} will just run two general policies: $\pi$ and $\pi^\prime$ for several episodes. 
\end{proof}

Below we analyze the properties of the two policies $\pi$ and $\pi^\prime$. 
Recall that $\pi$ is constructed as below.

\begin{equation}
\pi=\mathrm{argmin}_{\pi\in \Delta(\Pi_A\times\Pi_B)}\sup_{(\mu^\prime,\nu^\prime)\in\Pi_A\times\Pi_B}\sum_{h=1}^H\sum_{s,a,b}\frac{V^{\mu^\prime,\nu^\prime}(1_{h,s,a,b},P^{int})}{V^{\pi}(1_{h,s,a,b},P^{int})}.
\end{equation}

We begin with the uniform coverage property of the policy $\pi$ in Algorithm \ref{alg:fine}. 

\begin{lemma}\label{lem:uniform}
The policy $\pi$ in Algorithm \ref{alg:fine} satisfies that
$$\sup_{(\mu^\prime,\nu^\prime)\in\Pi_A\times\Pi_B}\sum_{h=1}^H\sum_{s,a,b}\frac{V^{\mu^\prime,\nu^\prime}(1_{h,s,a,b},P^{int})}{V^{\pi}(1_{h,s,a,b},P^{int})}\leq HSAB.$$
\end{lemma}

\begin{proof}[Proof of Lemma \ref{lem:uniform}]
    The proof is by plugging $\mathcal{X}=\{\{V^{\mu,\nu}(1_{h,\cdot,\cdot,\cdot},P^{int})\}_{h=1}^H \mid (\mu,\nu)\in\Pi_A\times\Pi_B\}$, $d=SAB$ and $m=H$ into Lemma \ref{lem:design}.
\end{proof}

Together with the relationship between $P^{int}$ and $\widetilde{P}$ (\eqref{equ:close}), we can replace the $P^{int}$ with $\widetilde{P}$.

\begin{lemma}\label{lem:uniform2}
    Conditioned on the high-probability case in Lemma \ref{lem6}, the policy $\pi$ in Algorithm \ref{alg:fine} satisfies that
$$\sup_{(\mu^\prime,\nu^\prime)\in\Pi_A\times\Pi_B}\sum_{h=1}^H\sum_{s,a,b}\frac{V^{\mu^\prime,\nu^\prime}(1_{h,s,a,b},\widetilde{P})}{V^{\pi}(1_{h,s,a,b},\widetilde{P})}\leq 12HSAB.$$
\end{lemma}

\begin{proof}[Proof of Lemma \ref{lem:uniform2}]
    For some pair of $(\mu^\prime,\nu^\prime)$, it holds that
\begin{equation}
    \begin{split}
    L.H.S &= \sum_{h=1}^H\sum_{s,a,b}\frac{V^{\mu^\prime,\nu^\prime}(1_{h,s,a,b},\widetilde{P})}{V^{\pi}(1_{h,s,a,b},\widetilde{P})} \\
    &\leq \sum_{h=1}^H\sum_{s,a,b}\frac{3V^{\mu^\prime,\nu^\prime}(1_{h,s,a,b},P^{int})}{\frac{1}{4}V^{\pi}(1_{h,s,a,b},P^{int})} \\&\leq 12HSAB,
    \end{split}
\end{equation}
where the first inequality is because of Lemma \ref{lem2}, the last inequality holds due to Lemma \ref{lem:uniform}.
\end{proof}

Now we state the properties of the auxiliary policy $\pi^\prime$. Recall that according to the main algorithm (Algorithm \ref{alg:main}), when the input infrequent set and intermediate transition kernel are $\mathcal{F}^k$ and $P^{int,k}$ respectively ($k=1,2$), we have the auxiliary policy $\pi^\prime=\pi_k$ and the number of episodes $T^\prime=NT^{(k)}$ ($N$ is the constant in Algorithm \ref{alg:main}). Due to the conclusion of Lemma \ref{lem:policy}, we have the following lemma for $\pi^\prime$ and $T^\prime$.

\begin{lemma}\label{lem:policy2}
    There exists a universal constant $N>0$, such that with probability at least $1-\frac{2S\delta}{K}$, for all possible $\mathcal{F}$, $\pi^\prime$, $T^\prime$ that appear in Algorithm \ref{alg:fine}, it holds that for all $(h,s,a,b,s^\prime)\notin \mathcal{F}$,
    $$T^\prime\cdot\mathbb{P}_{\pi^\prime}[(h,s,a,b,s^\prime)|P]\geq 3C_1 H^2\iota,$$
    where $C_1$ is the universal constant in Algorithm \ref{alg:crude}, $\mathbb{P}_{\pi}[(h,s,a,b,s^\prime)|P]$ means the probability of reaching $(s,a,b,s^\prime)$ at time step $h$ under policy $\pi$ and original MG $P$.
\end{lemma}

\begin{proof}[Proof of Lemma \ref{lem:policy2}]
According to Lemma \ref{lem:policy}, for a fixed pair of $\mathcal{F}^k$, $\pi_k$ and $T^{(k)}$, with probability at least $1-\frac{S\delta}{K}$, it holds that for all $(h,s,a,b,s^\prime)\notin \mathcal{F}^k$,
    $$T^{(k)}\cdot\mathbb{P}_{\pi_k}[(h,s,a,b,s^\prime)|P]\geq c_2H^2\iota.$$
Therefore, choosing $\pi^\prime=\pi_k$ and $T^\prime=NT^{(k)}$ with $N>\frac{3C_1}{c_2}$ implies that for a fixed pair of $\mathcal{F}^k$, $\pi_k$ and $T^{(k)}$, with probability at least $1-\frac{S\delta}{K}$, it holds that for all $(h,s,a,b,s^\prime)\notin \mathcal{F}^k$,
    $$NT^{(k)}\cdot\mathbb{P}_{\pi_k}[(h,s,a,b,s^\prime)|P]\geq 3C_1 H^2\iota.$$
    Finally, with a union bound over $k=1,2$, the proof is complete.
\end{proof}

The goal of running the auxiliary policy $\pi^\prime$ for $T^\prime$ episodes is to visit each state action pair not in the infrequent set for sufficient times. The property is summarized as below. 

\begin{lemma}\label{lem:count}
    Under the high-probability case in Lemma \ref{lem:policy2}, if we run policy $\pi^\prime$ for $T^\prime$ episodes, with probability $1-\frac{S\delta}{K}$, it holds that for all $(h,s,a,b,s^\prime)\notin \mathcal{F}$,
    $$N_h(s,a,b,s^\prime)\geq C_1 H^2\iota,$$
    where $N_h(s,a,b,s^\prime)$ is the visitation count of $(h,s,a,b,s^\prime)$. 
\end{lemma}

\begin{proof}[Proof of Lemma \ref{lem:count}]
    Under the high-probability case in Lemma \ref{lem:policy2}, for a fixed $(h,s,a,b,s^\prime)\notin \mathcal{F}$, 
    $$T^\prime\cdot\mathbb{P}_{\pi^\prime}[(h,s,a,b,s^\prime)|P]\geq 3C_1 H^2\iota.$$
    Then according to Lemma \ref{lem8}, with probability $1-\frac{\delta}{HSABK}$, it holds that
    $$N_h(s,a,b,s^\prime)\geq \frac{1}{2}T^\prime\cdot\mathbb{P}_{\pi^\prime}[(h,s,a,b,s^\prime)|P]-\iota \geq C_1 H^2\iota.$$
    Finally, the proof is complete due to a union bound over $(h,s,a,b,s^\prime)\notin \mathcal{F}$.
\end{proof}

Therefore, we could guarantee that with high probability, the refined transition kernel estimate $\widehat{P}$ is $\frac{1}{H}$-multiplicatively accurate to the absorbing MG $\widetilde{P}$, which is summarized in the lemma below.

\begin{lemma}\label{lem:close2}
    Under the high probability case in Lemma \ref{lem:count}, with probability $1-\frac{S\delta}{K}$, the output $\widehat{P}$ of Algorithm \ref{alg:fine} is $\frac{1}{H}$-multiplicatively accurate to $\widetilde{P}$.
\end{lemma}

\begin{proof}[Proof of Lemma \ref{lem:close2}]
    The proof is identical to the proof of Lemma \ref{lem6} (and Lemma E.3 in \citet{qiao2022sample}).
\end{proof}

It is known that the uncertainty of estimating $\widetilde{P}_h(\cdot|s,a,b)$ is proportional to $\sqrt{\frac{1}{N_h(s,a,b)}}$, where $N_h(s,a,b)$ is the visitation count of $(h,s,a,b)$ in the data set. Therefore, we prove the following lemma regarding $N_h(s,a,b)$'s.

\begin{lemma}\label{lem:keycount}
    Under the high probability case in Lemma \ref{lem:policy2}, if we run $\pi$ for $T$ episodes and $\pi^\prime$ for $T^\prime$ episodes, with probability $1-\frac{\delta}{K}$, for all $(h,s,a,b)\in[H]\times\mathcal{S}\times\mathcal{A}\times\mathcal{B}$, at least one of the following two statements hold: \\
    (1) For all $s^\prime\in\mathcal{S}$, $(h,s,a,b,s^\prime)\in\mathcal{F}$, i.e. $\widetilde{P}_h(s^\dagger|s,a,b)=1$ is fixed and known. \\
    (2) $N_h(s,a,b)\geq \frac{1}{2}TV^{\pi}(1_{h,s,a,b},\widetilde{P})$.
\end{lemma}

\begin{proof}[Proof of Lemma \ref{lem:keycount}]
    For any $(h,s,a,b)\in[H]\times\mathcal{S}\times\mathcal{A}\times\mathcal{B}$, if there exists some $s^\prime\in\mathcal{S}$, $(h,s,a,b,s^\prime)\notin\mathcal{F}$, then the expected visitation count of $(h,s,a,b)$ from running $\pi^\prime$ for $T^\prime$ episodes is $$E^1_{h,s,a,b} = T^\prime \mathbb{P}_{\pi^\prime}[(h,s,a,b)|P]\geq T^\prime\mathbb{P}_{\pi^\prime}[(h,s,a,b,s^\prime)|P] \geq 3C_1 H^2\iota, $$
    where the last inequality results from Lemma \ref{lem:policy2}.

    In addition, the expected visitation count of $(h,s,a,b)$ from running $\pi$ for $T$ episodes is $$E^2_{h,s,a,b}=TV^{\pi}(1_{h,s,a,b},P)\geq TV^{\pi}(1_{h,s,a,b},\widetilde{P}),$$
    where the inequality is because of Lemma \ref{rem10}.

    Combining the expectations, the total expected visitation count of $(h,s,a,b)$ is $$E_{h,s,a,b} = E^1_{h,s,a,b} +E^2_{h,s,a,b}\geq TV^{\pi}(1_{h,s,a,b},\widetilde{P})+3C_1 H^2\iota.$$

    According to Lemma \ref{lem8}, with probability $1-\frac{\delta}{HSABK}$, it holds that
    $$N_h(s,a,b)\geq \frac{1}{2}E_{h,s,a,b}-\iota\geq \frac{1}{2}TV^{\pi}(1_{h,s,a,b},\widetilde{P}).$$
    Finally, the proof is complete due to a union bound over all $(h,s,a,b)\in[H]\times\mathcal{S}\times\mathcal{A}\times\mathcal{B}$.
\end{proof}

Now we are ready to prove the following key lemma that bounds the difference between $\widehat{P}$ and $\widetilde{P}$. 

\begin{lemma}\label{lem:dif}
    Under the high-probability cases in Lemma \ref{lem:uniform2}, Lemma \ref{lem:close2} and Lemma \ref{lem:keycount}, with probability $1-\frac{S\delta}{K}$, the output $\widehat{P}$ of Algorithm \ref{alg:fine} satisfies that for all $(\mu,\nu)\in\Pi_A\times\Pi_B$ and any reward function $r^\prime$, 
    \begin{equation}
        \left|V^{\mu,\nu}(r^\prime,\widehat{P})-V^{\mu,\nu}(r^\prime,\widetilde{P})\right|\leq c_3 \sqrt{\frac{H^3 S^2 AB\iota}{T}}+c_3\frac{H^2 S^2 AB\iota}{T},
    \end{equation}
    where $c_3$ is a universal constant.
\end{lemma}

\begin{proof}[Proof of Lemma \ref{lem:dif}]
    Define $N_h(s,a,b)$ to be the visitation count of $(h,s,a,b)$ in the data set $\mathcal{D}$. Then according to Bernstein's inequality (Lemma \ref{lem3}), with probability $1-\frac{S\delta}{K}$, for all $(h,s,a,b,s^\prime)\in[H]\times\mathcal{S}\times\mathcal{A}\times\mathcal{B}\times\mathcal{S}$,
    \begin{equation}\label{equ:bernstein}
        \left|\widehat{P}_h(s^\prime|s,a,b)-\widetilde{P}_h(s^\prime|s,a,b)\right|\leq \sqrt{\frac{2\widetilde{P}_h(s^\prime|s,a,b)\iota}{N_h(s,a,b)}} +\frac{2\iota}{3N_h(s,a,b)}.
    \end{equation}
    Based on the above event and the high-probability cases in Lemma \ref{lem:uniform2}, Lemma \ref{lem:close2} and Lemma \ref{lem:keycount}, for any $(\mu,\nu)\in\Pi_A\times\Pi_B$ and reward function $r^\prime$, we define $\{f_h(\cdot)\}_{h=1}^{H+1}$ to be the value function of $(\mu,\nu)$ under $\widetilde{P}$ and $r^\prime$ ($f_{H+1}(\cdot)=0$). Then it holds that 
    \begin{equation}
    \begin{split}
    &\left|V^{\mu,\nu}(r^\prime,\widehat{P})-V^{\mu,\nu}(r^\prime,\widetilde{P})\right|\leq \sum_{h,s,a,b}V^{\mu,\nu}(1_{h,s,a,b},\widehat{P})\left|\left(\widehat{P}-\widetilde{P}\right)\cdot f_{h+1}(s,a,b)\right|\\\leq&
    \sum_{h,s,a,b}V^{\mu,\nu}(1_{h,s,a,b},\widehat{P})\cdot\left|\sum_{s^\prime}\left(\widehat{P}_h(s^\prime|s,a,b)- \widetilde{P}_h(s^\prime|s,a,b)\right)\cdot\left(f_{h+1}(s^\prime)-\widetilde{P}_h(\cdot|s,a,b)\cdot f_{h+1}(\cdot) \right)\right| \\\leq&
    \sum_{h,s,a,b}V^{\mu,\nu}(1_{h,s,a,b},\widehat{P})\cdot\sum_{s^\prime}\left(\sqrt{\frac{2\widetilde{P}_h(s^\prime|s,a,b)\iota}{N_h(s,a,b)}} +\frac{2\iota}{3N_h(s,a,b)}\right)\cdot\left|f_{h+1}(s^\prime)-\widetilde{P}_h(\cdot|s,a,b)\cdot f_{h+1}(\cdot)\right|\\\leq&
    4\sum_{h,s,a,b}V^{\mu,\nu}(1_{h,s,a,b},\widetilde{P})\cdot\left[\sqrt{\frac{2S\iota\cdot\text{Var}_{\widetilde{P}_h(\cdot|s,a,b)}f_{h+1}(\cdot)}{N_h(s,a,b)}}+\frac{HS\iota}{N_h(s,a,b)}\right] \\\leq&
    4\sqrt{2S\iota}\cdot\sqrt{\underbrace{\sum_{h,s,a,b}\frac{V^{\mu,\nu}(1_{h,s,a,b},\widetilde{P})}{N_h(s,a,b)}}_{(\mathrm{i})}}\cdot\sqrt{\underbrace{\sum_{h,s,a,b}V^{\mu,\nu}(1_{h,s,a,b},\widetilde{P})\text{Var}_{\widetilde{P}_h(\cdot|s,a,b)}f_{h+1}(\cdot)}_{(\mathrm{ii})}}+4HS\iota \underbrace{\sum_{h,s,a,b}\frac{V^{\mu,\nu}(1_{h,s,a,b},\widetilde{P})}{N_h(s,a,b)}}_{(\mathrm{i})}\\\leq&
    c_3 \sqrt{\frac{H^3 S^2 AB\iota}{T}}+c_3\frac{H^2 S^2 AB\iota}{T}.
    \end{split}
    \end{equation}
    The first inequality is because of simulation lemma (Lemma \ref{lem12}). The second inequality holds since $\widetilde{P}_h(\cdot|s,a,b)\cdot f_{h+1}(\cdot)$ is independent of $s^\prime$. The third inequality results from \eqref{equ:bernstein}. The forth inequality is derived via Cauchy-Schwarz inequality and Lemma \ref{lem:close2}. The fifth inequality is derived via Cauchy-Schwarz inequality. The last inequality comes from the inequalities \eqref{equ1} and \eqref{equ2} below.

    The upper bound for $\mathrm{(i)}$ is shown below: 
    \begin{equation}\label{equ1}
        \mathrm{(i)} = \sum_{h,s,a,b}\frac{V^{\mu,\nu}(1_{h,s,a,b},\widetilde{P})}{N_h(s,a,b)} \leq \sum_{h,s,a,b}\frac{V^{\mu,\nu}(1_{h,s,a,b},\widetilde{P})}{\frac{1}{2}TV^{\pi}(1_{h,s,a,b},\widetilde{P})}\leq \frac{24HSAB}{T}.
    \end{equation}
    The first inequality is because of Lemma \ref{lem:keycount} while the second inequality is due to Lemma \ref{lem:uniform2}. 

    The upper bound for $\mathrm{(ii)}$ is shown below:
    \begin{equation}\label{equ2}
        \mathrm{(ii)} = \sum_{h,s,a,b}V^{\mu,\nu}(1_{h,s,a,b},\widetilde{P})\text{Var}_{\widetilde{P}_h(\cdot|s,a,b)}f_{h+1}(\cdot) \leq H^2,
    \end{equation}
    where the inequality results from a recursive application of Law of Total Variance.
\end{proof}

\section{Proof of main theorems}\label{appc}
\subsection{Proof of Theorem~\ref{thm:main}}
We first give a proof for the upper bound on the number of stages.
\begin{lemma}\label{lem50}
If $T^{(k)}=K^{1-\frac{1}{2^{k}}}$ for $k=1,2\cdots$, we have 
$$K_{0}\leq\min\{j:\sum_{k=1}^{j}T^{(k)}\geq K\}=O(\log\log K).$$
\end{lemma}

\begin{proof}[Proof of Lemma~\ref{lem50}]
    Take $j=\log_{2}\log_{2}K$, we have $T^{(j)}=\frac{K}{K^{(\log_{2}K)^{-1}}}=\frac{K}{2}$, which means that $$K_{0}\leq \log_{2}\log_{2}K+2=O(\log\log K).$$
\end{proof}

Then we are able to bound the batch complexity of Algorithm~\ref{alg:main}.
\begin{lemma}\label{lem15}
	The batch complexity of Algorithm~\ref{alg:main} is bounded by $O(H + \log\log K)$.
\end{lemma}

\begin{proof}[Proof of Lemma~\ref{lem15}]
    According to Lemma~\ref{lem1} and Lemma~\ref{lem11}, the batch complexity for each of the first two stages is bounded by $H+2$, while the batch complexity for each of the remaining stages is bounded by $2$. Therefore the total batch complexity is bounded by $2H+2K_0=O(H + \log\log K)$.
\end{proof}
 
 Recall that in Algorithm~\ref{alg:main}, in each stage $k$ ($k=1,2$), we run Algorithm~\ref{alg:crude} to construct the infrequent tuples $\mathcal{F}^{k}$ and the intermediate transition kernel $P^{int,k}$. The absorbing MG $\widetilde{P}^{k}$ is constructed as in Definition~\ref{def2} based on $\mathcal{F}^{k}$ and the real MG $P$. Then we run Algorithm~\ref{alg:fine} to construct an empirical estimate of $\widetilde{P}^{k}$, which is $\widehat{P}^{k}$. After the first two stages, for the following stages ($k=3,4,\cdots,K_0$), we keep the infrequent tuples $\mathcal{F}^{2}$, the absorbing MG $\widetilde{P}^2$ and the intermediate transition kernel $P^{int,2}$. In each stage, we run Algorithm~\ref{alg:fine} to construct an empirical estimate of $\widetilde{P}^{2}$, which is $\widehat{P}^{k}$.

\begin{lemma}\label{lem:valuedif}
There exists a universal constant $C>0$, such that with probability $1-\delta$, for any $k=1,2$ and $(\mu,\nu)\in\Pi_A^{k}\times\Pi_B$, 
	$$|V^{\mu,\nu}(r,\widehat{P}^{k})-V^{\mu,\nu}(r,P)|\leq C\left(\sqrt{\frac{H^{3}S^{2}AB\iota}{T^{(k)}}}+\frac{H^{5}S^{3}A^{2}B^2\iota}{T^{(k)}}\right),$$
	while for any $k\geq 3$ and $(\mu,\nu)\in\Pi_A^{k}\times\Pi_B$, 
	$$|V^{\mu,\nu}(r,\widehat{P}^{k})-V^{\mu,\nu}(r,P)|\leq C\left(\sqrt{\frac{H^{3}S^{2}AB\iota}{T^{(k)}}}+\frac{H^{5}S^{3}A^{2}B^2\iota}{T^{(2)}}\right).$$
\end{lemma}

\begin{proof}[Proof of Lemma \ref{lem:valuedif}]
    Conditioned on the high probability case in Lemma \ref{lemfirst} (for $k=1,2$), it holds that for any $k=1,2$ and $(\mu,\nu)\in\Pi_A^{k}\times\Pi_B$, 
    $$\left| V^{\mu,\nu}(r,P)-V^{\mu,\nu}(r,\widetilde{P}^k)\right|\leq \frac{c_1 H^5 S^{3}A^{2}B^2\iota}{T^{(k)}},$$
    where $c_1$ is a universal constant and $r$ is the real reward function.

    In addition, conditioned on the high probability case in Lemma \ref{lem:dif} (for $k=1,2,\cdots,K_0$), it holds that for $k=1$ and any $(\mu,\nu)\in\Pi_A^{1}\times\Pi_B$,
    $$\left|V^{\mu,\nu}(r,\widehat{P}^1)-V^{\mu,\nu}(r,\widetilde{P}^1)\right|\leq c_3 \sqrt{\frac{H^3 S^2 AB\iota}{T^{(1)}}}+c_3\frac{H^2 S^2 AB\iota}{T^{(1)}},$$
    while for any $k\geq2$ and $(\mu,\nu)\in\Pi_A^{k}\times\Pi_B$,
    $$\left|V^{\mu,\nu}(r,\widehat{P}^k)-V^{\mu,\nu}(r,\widetilde{P}^2)\right|\leq c_3 \sqrt{\frac{H^3 S^2 AB\iota}{T^{(k)}}}+c_3\frac{H^2 S^2 AB\iota}{T^{(k)}}.$$
    By combining the above inequalities and choosing $C=c_1+c_3$, the conclusion holds.

    The failure probability is bounded by $O\left(\frac{S\delta}{K}\right)\times O(\log\log K)\leq \delta,$ if $K\geq \widetilde{\Omega}(S)$.
\end{proof}

Based on the above high-probability case, the estimations of value functions are uniformly accurate for all policy pairs, which implies a sub-optimality bound for the remaining policies in the version space.

\begin{lemma}\label{lemj4}
	With probability $1-\delta$, for the policies that have not been eliminated at stage $k$, i.e. $\mu\in\Pi_A^{k+1}$, we have that
	$$V^\star(r,P)-V^{\mu,\dagger}(r,P)\leq 4C\left(\sqrt{\frac{H^{3}S^{2}AB\iota}{T^{(k)}}}+\frac{H^{5}S^{3}A^{2}B^2\iota}{T^{(k)}}\right),\ k=1,2.$$
	$$V^{\star}(r,P)-V^{\mu,\dagger}(r,P)\leq 4C\left(\sqrt{\frac{H^{3}S^{2}AB\iota}{T^{(k)}}}+\frac{H^{5}S^{3}A^{2}B^2\iota}{T^{(2)}}\right),\ k\geq 3.$$
\end{lemma}

\begin{proof}[Proof of Lemma \ref{lemj4}]
    We only prove the case for $k=1$, the remaining cases can be proven similarly.

    We prove under the high probability case in Lemma \ref{lem:valuedif}, which says that for any $(\mu,\nu)\in\Pi_A^{1}\times\Pi_B$, 
	$$|V^{\mu,\nu}(r,\widehat{P}^{1})-V^{\mu,\nu}(r,P)|\leq C\left(\sqrt{\frac{H^{3}S^{2}AB\iota}{T^{(1)}}}+\frac{H^{5}S^{3}A^{2}B^2\iota}{T^{(1)}}\right).$$

Recall that $\mu^\star$ is the Nash policy of the max-player. Then it holds that 
 \begin{align*}
	&\sup_{\mu\in\Pi_A^{1}}\inf_{\nu\in\Pi_B} V^{\mu,\nu}(r,\widehat{P}^{1})-\inf_{\nu\in\Pi_B} V^{\mu^\star,\nu}(r,\widehat{P}^{1})=\inf_{\nu\in\Pi_B} V^{\widehat{\mu},\nu}(r,\widehat{P}^{1})-\inf_{\nu\in\Pi_B} V^{\mu^\star,\nu}(r,\widehat{P}^{1})\\ \leq& \left|\inf_{\nu\in\Pi_B} V^{\widehat{\mu},\nu}(r,\widehat{P}^{1})-\inf_{\nu\in\Pi_B} V^{\widehat{\mu},\nu}(r,P)\right|+\inf_{\nu\in\Pi_B} V^{\widehat{\mu},\nu}(r,P)-\inf_{\nu\in\Pi_B} V^{\mu^\star,\nu}(r,P)\\&+\left|\inf_{\nu\in\Pi_B} V^{\mu^\star,\nu}(r,P)-\inf_{\nu\in\Pi_B} V^{\mu^\star,\nu}(r,\widehat{P}^1)\right|\\\leq& 2C\left(\sqrt{\frac{H^{3}S^{2}AB\iota}{T^{(1)}}}+\frac{H^{5}S^{3}A^{2}B^2\iota}{T^{(1)}}\right),
\end{align*}
where the last inequality holds due to Lemma \ref{lem:minmax}.

Therefore, the Nash policy $\mu^\star$ will not be eliminated. In addition, for $\mu\in\Pi_A^2$,
\begin{align*}
\inf_{\nu\in\Pi_B}V^{\mu^\star,\nu}(r,\widehat{P}^{1})-\inf_{\nu\in\Pi_B}V^{\mu,\nu}(r,\widehat{P}^{1}) \leq&\sup_{\mu\in\Pi_A^{1}}\inf_{\nu\in\Pi_B} V^{\mu,\nu}(r,\widehat{P}^{1})-\inf_{\nu\in\Pi_B} V^{\mu,\nu}(r,\widehat{P}^{1})\\\leq& 2C\left(\sqrt{\frac{H^{3}S^{2}AB\iota}{T^{(1)}}}+\frac{H^{5}S^{3}A^{2}B^2\iota}{T^{(1)}}\right).
\end{align*}
Finally, together with Lemma \ref{lem:minmax}, it holds that
\begin{align*}
	&V^{\star}(r,P)-V^{\mu,\dagger}(r,P)\\\leq& \left|V^{\star}(r,P)-\inf_{\nu\in\Pi_B} V^{\mu^\star,\nu}(r,\widehat{P}^{1})\right|+\inf_{\nu\in\Pi_B} V^{\mu^\star,\nu}(r,\widehat{P}^{1})-\inf_{\nu\in\Pi_B}V^{\mu,\nu}(r,\widehat{P}^{1})+\left|\inf_{\nu\in\Pi_B}V^{\mu,\nu}(r,\widehat{P}^{1})-\inf_{\nu\in\Pi_B}V^{\mu,\nu}(r,P)\right|\\ \leq& 4C\left(\sqrt{\frac{H^{3}S^{2}AB\iota}{T^{(1)}}}+\frac{H^{5}S^{3}A^{2}B^2\iota}{T^{(1)}}\right).
\end{align*}

The conclusions for the remaining stages can be proven by identical techniques and induction.
\end{proof}

Given the sub-optimality bound for the policies, we are ready to provide the final regret bound.

\begin{lemma}\label{lem19}
	Conditioned on the high probability event in Lemma~\ref{lem:valuedif}, if $K\geq\widetilde{\Omega}(H^{14} S^8 A^6 B^6)$, the total regret is bounded by $\widetilde{O}(\sqrt{H^{2}S^{2}ABT})$, where $T:=HK$ is the number of steps.
\end{lemma}

\begin{proof}[Proof of Lemma~\ref{lem19}]
	The regret for the first stage (stage 1) is at most $(N+2)HT^{(1)}=O(HK^{\frac{1}{2}})$. \\
	Because of Lemma~\ref{lemj4}, the regret for the second stage (stage 2) is at most $$(N+2)T^{(2)}\times 4C\left(\sqrt{\frac{H^{3}S^{2}AB\iota}{T^{(1)}}}+\frac{H^{5}S^{3}A^{2}B^2\iota}{T^{(1)}}\right).$$
	For stage $k\geq3$, since the remaining policies (any $\mu\in\Pi_A^{k}$) are at most $4C\left(\sqrt{\frac{H^{3}S^{2}AB\iota}{T^{(k-1)}}}+\frac{H^{5}S^{3}A^{2}B^2\iota}{T^{(2)}}\right)$ sub-optimal and the policy $\pi$ is a mixture of remaining policies, the regret due to running $\pi$ for $T$ episodes is at most $4CT^{(k)}\left(\sqrt{\frac{H^{3}S^{2}AB\iota}{T^{(k-1)}}}+\frac{H^{5}S^{3}A^{2}B^2\iota}{T^{(2)}}\right)$. Besides, the regret due to running $\pi_2$ for $NT^{(2)}$ episodes is bounded by 
    $NT^{(2)}\times 4C\left(\sqrt{\frac{H^{3}S^{2}AB\iota}{T^{(1)}}}+\frac{H^{5}S^{3}A^{2}B^2\iota}{T^{(1)}}\right)$. Combining the results, the regret for the k-th stage is at most
	$O\left(T^{(k)}\left(\sqrt{\frac{H^{3}S^{2}AB\iota}{T^{(k-1)}}}+\frac{H^{5}S^{3}A^{2}B^2\iota}{T^{(2)}}\right)+T^{(2)}\left(\sqrt{\frac{H^{3}S^{2}AB\iota}{T^{(1)}}}+\frac{H^{5}S^{3}A^{2}B^2\iota}{T^{(1)}}\right)\right).$ 
 
Adding up the regret for each stage , we have that the total regret is bounded by
	\begin{align*}
	\text{Regret}(K)&\leq O(HK^{\frac{1}{2}})+O\left(\sum_{k=2}^{K_{0}}T^{(k)}\sqrt{\frac{H^{3}S^{2}AB\iota}{T^{(k-1)}}}\right)+O\left(K\cdot\frac{H^5 S^3 A^2 B^2 \iota}{T^{(2)}}\right)\\&+O\left(T^{(2)}\left(\sqrt{\frac{H^{3}S^{2}AB\iota}{T^{(1)}}}+\frac{H^{5}S^{3}A^{2}B^2\iota}{T^{(1)}}\right)\right)\cdot O(\log\log K)\\ &= O(HK^{\frac{1}{2}})+ O(\sqrt{H^{3}S^{2}ABK\iota}\cdot \log\log K)+O(H^{5}S^{3}A^{2}B^2K^{\frac{1}{4}}\iota\cdot\log\log K)\\ &= \widetilde{O}(\sqrt{H^{3}S^{2}ABK}),
	\end{align*}
where the last equality is because $K\geq\widetilde{\Omega}(H^{14} S^8 A^6 B^6)$.
\end{proof}

Then Theorem~\ref{thm:main} holds because of Lemma~\ref{lem15}, Lemma~\ref{lem:valuedif} and Lemma~\ref{lem19}.

\subsection{Proof of Theorem \ref{thm:lower}}
Since two-player zero-sum MG strictly generalizes single-agent MDPs, the lower bound directly follows from Theorem 2 of \citet{zhang2022near}.

\subsection{Proof of Theorem \ref{thm:sample}}
We can output any policy $\mu$ in the remaining policy set $\Pi_A^{K_{0}+1}$ ($\mu$ is a Markov policy). Because there are $K_{0}=O(\log\log K)$ stages in total, the maximal $T^{(k)}$ is larger than $\Omega(\frac{K}{\log\log K})$. According to Lemma~\ref{lemj4}, for any $\mu\in\Pi_A^{K_{0}+1}$, $$V^{\star}(r,P)-V^{\mu,\dagger}(r,P)\leq 4C\left(\sqrt{\frac{H^{3}S^{2}AB\iota}{T^{(k)}}}+\frac{H^{5}S^{3}A^{2}B^2\iota}{T^{(2)}}\right),\,\forall\,k=2,3,\cdots,K_{0}.$$
Combining these two results, we have for any $\mu\in\Pi_A^{K_{0}+1}$,
$$V^{\star}(r,P)-V^{\mu,\dagger}(r,P)\leq \widetilde{O}\left(\sqrt{\frac{H^{3}S^{2}AB}{K}}\right).$$
Then $K=\widetilde{O}\left(\frac{H^{3}S^{2}AB}{\epsilon^{2}}\right)$ bounds the above by $\epsilon$.

\section{Proof for the results in Section \ref{sec:dis}}
\subsection{Details for the bandit game setting}\label{sec:bandit}
Recall that bandit game is a special case of Markov games where $H=S=1$. The action set for the max-player is $\mathcal{A}$ ($A:=|\mathcal{A}|$) and the action set for the min-player is $\mathcal{B}$ ($B:=|\mathcal{B}|$). Therefore, a policy $\mu$ for the max-player is a distribution over $\mathcal{A}$ while a policy $\nu$ for the min-player is a distribution over $\mathcal{B}$. For simplicity, here we denote the reward function by $r(a,b)$. In addition, the value function under policy pair $(\mu,\nu)$ and the Nash value are denoted by $r(\mu,\nu)$ and $r^\star$ respectively.  The goal is to minimize the regret for the max-player:

$$\text{Regret}(K)=\sum_{k=1}^K \left[r^\star-r(\mu^k,\dagger)\right]=\sum_{k=1}^K \left[r^\star-\inf_{\nu}r(\mu^k,\nu)\right].$$

Now we apply Algorithm \ref{alg:main} (with some revision) to the bandit game setting. The main difference is that now we do not need to construct the absorbing model since we can directly try each action. The detailed algorithm is shown below.

\begin{algorithm}[tbh]
	\caption{Algorithm for bandit games}\label{alg:bandit}
	\begin{algorithmic}[1]
		\STATE \textbf{Require}: Number of episodes $K$. Universal constant $C$. Failure probability $\delta$.
		\STATE \textbf{Initialize}: $T^{(k)}=K^{1-\frac{1}{2^{k}}}$, $k\leq K_{0}=O(\log\log K)$. $\Pi_A^{1}:=\{ \text{All stochastic policies for the max-player}\}$, $\mu_0:=\text{uniform distribution over } \mathcal{A}$, $\nu_0:=\text{uniform distribution over } \mathcal{B}$. $\iota=\log(2ABK/\delta)$, $T^0=2AB\iota$. Data set $\mathcal{D}=\emptyset$.
		\FOR{$k=1,2,\cdots,K_{0}$}  
		\STATE \maroon{$\diamond$ Number of episodes in $k$-th stage:}
		\IF{$\sum_{i=1}^{k}T^{(i)}+kT^0\geq K$} 
		\STATE $T^{(k)}=K-kT^0-\sum_{i=1}^{k-1}T^{(i)}$ (o.w. $T^{(k)}=K^{1-\frac{1}{2^{k}}}$).
		\ENDIF
		\ENDFOR
		\FOR{$k=1,2,\cdots,K_0$}
		\STATE \maroon{$\diamond$ Construct an explorative policy for the max-player:}
		\STATE Construct policy $\mu=\mathrm{argmin}_{\mu\in\Pi_A^k} \sup_{\mu^\prime\in\Pi_A^k}\sum_{a\in\mathcal{A}}\frac{\mu^\prime(a)}{\mu(a)}$. 
           \STATE \maroon{$\diamond$ Run the explorative policy and the uniform policy for several episodes:}
		\STATE Run $\pi=(\mu,\nu_0)$ for $T^{(k)}$ episodes and store the rewards in $\mathcal{D}$.
            \STATE Run $\pi_0=(\mu_0,\nu_0)$ for $T^0$ episodes and store the rewards in $\mathcal{D}$.
            \STATE Construct an empirical reward function $\widehat{r}^k(a,b)$ (and therefore $\widehat{r}^k(\mu,\nu)$) using $\mathcal{D}$.
            \STATE \maroon{$\diamond$ Policy elimination for the max-player:}
		\STATE $\Pi_A^{k+1}\leftarrow\left\{\mu\in\Pi_A^k\mid \min_{b\in\mathcal{B}} \widehat{r}^k(\mu,b)\geq \sup_{\mu\in\Pi_A^{k}}\min_{b\in\mathcal{B}} \widehat{r}^k(\mu,b)-2C\left(\sqrt{\frac{AB\iota}{T^{(k)}}}\right)\right\}$.
            \STATE \maroon{$\diamond$ Reset the data set:}
		\STATE Clear the data set: $\mathcal{D}\leftarrow\emptyset$.
		\ENDFOR
	\end{algorithmic}
\end{algorithm}

Now we restate the upper bounds and give a proof.
\begin{theorem}[Restate Theorem \ref{thm:bandit}]\label{thm:re1}
If we run Algorithm \ref{alg:bandit} (adapted from Algorithm \ref{alg:main}) under a two-player zero-sum bandit game for $K$ episodes, the batch complexity is bounded by $O(\log\log K)$ while the regret is $\widetilde{O}(\sqrt{ABK})$ with high probability. Furthermore, the algorithm is computationally efficient.
\end{theorem}

\begin{proof}[Proof of Theorem \ref{thm:re1}]
    Using identical proof of Lemma \ref{lem50}, we have the number of stages $K_{0}\leq O(\log\log K)$. In addition, in each stage, we only run two policies $\pi$ and $\pi_0$, which can be done in two batches. Therefore, the batch complexity is bounded by $2K_0=O(\log\log K)$.

    For the regret part, we first have the following property of the $\mu^k$ constructed in $k$-th stage: for any $\mu^\prime\in\Pi_A^k$, $$\sum_{a\in\mathcal{A}}\frac{\mu^\prime(a)}{\mu^k(a)}\leq A,$$
    which follows from Lemma \ref{lem:design}. In addition, using identical proof of Lemma \ref{lem:keycount}, it holds that with probability $1-\delta/2$, for all $(k,a,b)\in[K_0]\times\mathcal{A}\times\mathcal{B}$,
    $$N^k(a,b)\geq \frac{T^{(k)}}{2}\cdot\frac{\mu^k(a)}{B},$$
    where $N^k(a,b)$ is the visitation count of $(a,b)$ in the data set $\mathcal{D}$ from the $k$-th stage.

    Now we are ready to bound the estimation error of $\widehat{r}^k$. With probability $1-\delta$, uniformly for all $(k,\mu,b)\in[K_0]\times\Pi_A^k\times\mathcal{B}$,
\begin{align*}
&\left|\widehat{r}^k(\mu,b)-r(\mu,b)\right|\leq \sum_{a\in\mathcal{A}}\mu(a)|\widehat{r}^k(a,b)-r(a,b)|\\\leq&
C\sum_{a\in\mathcal{A}}\mu(a)\sqrt{\frac{\iota}{N^k(a,b)}}\leq C\sum_{a\in\mathcal{A}}\mu(a)\sqrt{\frac{B\iota}{T^{(k)}\mu^k(a)}}\\\leq&
C\sqrt{B\iota\sum_{a\in\mathcal{A}}\mu(a)}\cdot\sqrt{\sum_{a\in\mathcal{A}}\frac{\mu(a)}{T^{(k)}\mu^k(a)}}\leq C\sqrt{\frac{AB\iota}{T^{(k)}}},
\end{align*}
where the second inequality holds for some universal constant $C$ with probability $1-\delta/2$ due to Hoeffding's inequality. The forth inequality holds due to Cauchy-Schwarz inequality.

Based on the uniform estimation error, the regret bound of $\widetilde{O}(\sqrt{ABK})$ can be derived using identical proof as Lemma \ref{lemj4} and Lemma \ref{lem19}.

Lastly, we deal with the computational complexity. Note that any policy $\mu$ for the max-player can be represented as a vector $\mu=(p_1,\cdots,p_A)$ such that $\sum_{i=1}^A p_i=1$. We prove by induction that the constrains on the remaining policy set are at most $O(B\log\log K)$ linear constraints and all steps are efficient. Assume the induction assumption is true at the beginning of the $k$-th stage, then $\mu=\mathrm{argmin}_{\mu\in\Pi_A^k} \sup_{\mu^\prime\in\Pi_A^k}\sum_{a\in\mathcal{A}}\frac{\mu^\prime(a)}{\mu(a)}$ can be constructed (approximately) in $\mathrm{poly}(A,B,K)$ time according to Lemma \ref{lem:design}. Then given the data set $\mathcal{D}^k$, we do policy elimination.
Recall that we will keep the policy $\mu$ if $\min_{b\in\mathcal{B}} \widehat{r}^k(\mu,b)\geq \sup_{\mu\in\Pi_A^{k}}\min_{b\in\mathcal{B}} \widehat{r}^k(\mu,b)-2C\left(\sqrt{\frac{AB\iota}{T^{(k)}}}\right)$. Now we consider the computation.

The computation of the R.H.S is equivalent to solving\\ $\max x$ s.t. $x\leq\sum_{i=1}^A p_i\widehat{r}^k(i,b)$ for all $b\in\mathcal{B}$ and the previous conditions on $(p_1,\cdots,p_A)$.\\
According to the induction assumption, the previous conditions are at most $O(B\log\log K)$ linear constraints. Therefore the problem is a linear programming (LP) problem with at most $A+1$ variables and $O(B\log\log K)$ linear constraints, which can be solved in polynomial time \citep{nemhauser1988polynomial}. With the value $c$ of R.H.S, the elimination step is equivalent to adding the following $B$ linear constraints to the policy set: $\sum_{i=1}^A p_i\widehat{r}^k(i,b)\geq c$ for all $b\in\mathcal{B}$. Since there are $O(\log\log K)$ stages in total, the total number of linear constraints is bounded by $O(B\log\log K)$.

As a result, the induction assumption always holds and the computation is efficient.

\end{proof}

\subsection{Proof for Theorem \ref{thm:rfe}}\label{sec:rfe}
In this part, we restate Theorem \ref{thm:rfe} and provide a proof.

\begin{theorem}[Restate Theorem \ref{thm:rfe}]\label{thm:re2}
The batch complexity of Algorithm \ref{alg:rfe} is bounded by $H+2$. There exists a constant $c>0$ such that, for any $\epsilon >0$ and any $\delta>0$, if the number of total episodes $K$ satisfies that $$K>c\left(\frac{H^3 S^{2}AB \iota^{\prime}}{\epsilon^{2}}+\frac{H^{5}S^3 A^2 B^2 \iota^{\prime}}{\epsilon}\right),$$ where $\iota^{\prime}=\log(\frac{HSAB}{\epsilon\delta})$, then there exists a choice of $N_{0}$ and $N_1$ such that $N_{0}+N_1=K$ and with probability $1-\delta$, for any reward function $r$, Algorithm~\ref{alg:rfe} will output a policy pair $\pi^{r}=(\mu^r,\nu^r)$ that is $\epsilon$-approximate Nash.
\end{theorem}

\begin{proof}[Proof of Lemma \ref{thm:re2}]
    We first prove the batch complexity upper bound. According to Lemma~\ref{lem1} and Lemma~\ref{lem11}, the batch complexity of Algorithm \ref{alg:rfe} is bounded by $H+2$.

    For the sample complexity, due to Lemma \ref{lemfirst}, with probability $1-\delta/2$, for any policy pair $(\mu,\nu)\in\Pi_A\times\Pi_B$ and reward function $r$, it holds that (for some universal constant $c_1$)
  $$\left|V^{\mu,\nu}(r,P)-V^{\mu,\nu}(r,\widetilde{P})\right|\leq \frac{c_1 H^5 S^{3}A^{2}B^2\iota}{N_0}.$$
   In addition, resulting from Lemma \ref{lem:dif}, with probability $1-\delta/2$, for all $(\mu,\nu)\in\Pi_A\times\Pi_B$ and any reward function $r$, it holds that (for some universal constant $c_3$)
    \begin{equation}
        \left|V^{\mu,\nu}(r,\widehat{P})-V^{\mu,\nu}(r,\widetilde{P})\right|\leq c_3 \sqrt{\frac{H^3 S^2 AB\iota}{N_1}}+c_3\frac{H^2 S^2 AB\iota}{N_1}.
    \end{equation}
    Combining the results, we have with probability $1-\delta$, for all $(\mu,\nu)\in\Pi_A\times\Pi_B$ and any reward function $r$,
    \begin{equation}\label{equu}
        \left|V^{\mu,\nu}(r,\widehat{P})-V^{\mu,\nu}(r,P)\right|\leq \frac{c_1 H^5 S^{3}A^{2}B^2\iota}{N_0}+ c_3 \sqrt{\frac{H^3 S^2 AB\iota}{N_1}}+c_3\frac{H^2 S^2 AB\iota}{N_1}.
    \end{equation}
    Therefore there exists some universal constant $c$ such that by choosing $N_0=\frac{cH^5 S^3 A^2 B^2\iota^\prime}{4\epsilon}$ and $N_1=\frac{cH^3 S^{2}AB \iota^{\prime}}{\epsilon^{2}}+\frac{cH^{2}S^2 AB \iota^{\prime}}{2\epsilon}$, the R.H.S of \eqref{equu} is bounded by $\epsilon/2$. In this case, the total sample complexity is $
    \frac{cH^5 S^3 A^2 B^2\iota^\prime}{2\epsilon}+\frac{cH^3 S^{2}AB \iota^{\prime}}{\epsilon^{2}}+\frac{cH^{2}S^2 AB \iota^{\prime}}{2\epsilon}\leq c\left(\frac{H^3 S^{2}AB \iota^{\prime}}{\epsilon^{2}}+\frac{H^{5}S^3 A^2 B^2 \iota^{\prime}}{\epsilon}\right)$.
    
    Finally, we prove that $\mu^r=\mathrm{argmax}_{\mu\in\Pi_A}\inf_{\nu\in\Pi_B}V^{\mu,\nu}(r,\widehat{P})$ is $\epsilon$-approximate Nash. Due to Lemma \ref{lem:minmax}, we have
    \begin{align*}
        &V^\star(r,P)-V^{\mu^r,\dagger}(r,P)=V^\star(r,P)-\inf_{\nu\in\Pi_B}V^{\mu^r,\nu}(r,P)\\\leq&
        \left|\inf_{\nu\in\Pi_B}V^{\mu^\star,\nu}(r,P)-\inf_{\nu\in\Pi_B}V^{\mu^\star,\nu}(r,\widehat{P})\right|+\inf_{\nu\in\Pi_B}V^{\mu^\star,\nu}(r,\widehat{P})-\inf_{\nu\in\Pi_B}V^{\mu^r,\nu}(r,\widehat{P})\\&+\left|\inf_{\nu\in\Pi_B}V^{\mu^r,\nu}(r,\widehat{P})-\inf_{\nu\in\Pi_B}V^{\mu^r,\nu}(r,P)\right|\\\leq&
        \frac{\epsilon}{2}+0+\frac{\epsilon}{2}=\epsilon.
    \end{align*}
    The conclusion that $\nu^r$ is also $\epsilon$-approximate Nash is proven by symmetry.
\end{proof}

\section{Technical lemmas}\label{appt}
\begin{lemma}[Multiplicative Chernoff bound \citep{chernoff1952measure}]\label{lem:chernoff}
Let $X$ be a Binomial random variable with parameter $p, n$. For any $\delta \in [0,1]$, we have that $\mathbb{P}[X > (1 + \delta)pn] < \left(\frac{e^{\delta}}{(1+\delta)^{1+\delta}}\right)^{np}$.
A slightly looser bound that suffices for our propose:
$$ \mathbb{P}[X > (1 + \delta)pn] < e^{-\frac{\delta^2 pn}{3}}.$$
\end{lemma}

\begin{lemma}[Bernstein's inequality]\label{lem3}
	Let $x_{1},\cdots,x_{n}$ be independent bounded random variables such that $\mathbb{E}[x_{i}]=0$ and $|x_{i}|\leq A$ with probability $1$. Let $\sigma^{2}=\frac{1}{n}\sum_{i=1}^{n}\mathrm{Var}[x_{i}]$, then with probability $1-\delta$ we have 
	$$\left|\frac{1}{n}\sum_{i=1}^{n}x_{i}\right|\leq \sqrt{\frac{2\sigma^{2}\log(2/\delta)}{n}}+\frac{2A}{3n}\log(2/\delta).$$
\end{lemma}

\begin{lemma}[Empirical Bernstein's inequality \citep{maurer2009empirical}]\label{lem4}
	Let $x_{1},\cdots,x_{n}$ be i.i.d random variables such that $|x_{i}|\leq A$ with probability $1$. Let $\overline{x}=\frac{1}{n}\sum_{i=1}^{n}x_{i}$, and $\widehat{V}_{n}=\frac{1}{n}\sum_{i=1}^{n}(x_{i}-\overline{x})^{2}$, then with probability $1-\delta$ we have 
	$$\left|\frac{1}{n}\sum_{i=1}^{n}x_{i}-\mathbb{E}[x]\right|\leq \sqrt{\frac{2\widehat{V}_{n}\log(2/\delta)}{n}}+\frac{7A}{3n}\log(2/\delta).$$
\end{lemma}

\begin{lemma}[Lemma F.4 in \citep{dann2017unifying}]\label{lem7}
Let $F_{i}$ for $i = 1,\cdots$ be a filtration and $X_{1},\cdots, X_{n}$ be a sequence of Bernoulli random variables with $\mathbb{P}(X_{i} = 1|F_{i-1}) = P_{i}$ with $P_{i}$ being $F_{i-1}$-measurable and $X_{i}$ being $F_{i}$ measurable. It holds that
$$\mathbb{P}\left[\exists\, n: \sum_{t=1}^{n}X_{t} < \sum_{t=1}^{n}P_{t}/2-W \right]\leq e^{-W}.$$	
\end{lemma}

\begin{lemma}\label{lem8}
	Let $F_{i}$ for $i = 1,\cdots$ be a filtration and $X_{1},\cdots, X_{n}$ be a sequence of Bernoulli random variables with $\mathbb{P}(X_{i} = 1|F_{i-1}) = P_{i}$ with $P_{i}$ being $F_{i-1}$-measurable and $X_{i}$ being $F_{i}$ measurable. It holds that
	$$\mathbb{P}\left[\exists\, n: \sum_{t=1}^{n}X_{t} < \sum_{t=1}^{n}P_{t}/2-\iota \right]\leq \frac{\delta}{HSABK},$$	where $\iota=\log(2HSABK/\delta)$.
\end{lemma}

\begin{proof}[Proof of Lemma~\ref{lem8}]
	Directly plug in $W=\iota$ in lemma~\ref{lem7}.
\end{proof}

\begin{lemma}[Lemma 1 in \citet{zhang2022near}]\label{lem:design}
    Let $d>0$ be an integer. Let $\mathcal{X}\subset(\Delta^d)^m$. Then there exists a distribution $\mathcal{D}$ over $\mathcal{X}$, such that
    $$\max_{x=\{x_i\}_{i=1}^{dm}\in\mathcal{X}}\sum_{i=1}^{dm}\frac{x_i}{y_i}=md,$$
    where $y=\{y_i\}_{i=1}^{dm}=\mathbb{E}_{x\sim\mathcal{D}}[x]$. Moreover, if $\mathcal{X}$ has a boundary set $\partial \mathcal{X}$ with finite cardinality, we can find an approximation solution for $\mathcal{D}$ in poly($|\partial\mathcal{X}|$) time.
\end{lemma}

\begin{lemma}[Simulation lemma \citep{dann2017unifying}]\label{lem12}
	For any two MGs $M^{\prime}$ and $M^{\prime\prime}$ with rewards $r^{\prime}$ and $r^{\prime\prime}$ and transition probabilities $\mathcal{P}^{\prime}$ and $\mathcal{P}^{\prime\prime}$, the difference in values $V^{\prime}$, $V^{\prime\prime}$ with respect to the same policy pair $(\mu,\nu)$ can be written as 
	$$V_{h}^{\prime}(s)-V_{h}^{\prime\prime}(s)=\mathbb{E}_{M^{\prime\prime},\mu,\nu}\left[\sum_{i=h}^{H}[r_{i}^{\prime}(s_{i},a_{i},b_i)-r_{i}^{\prime\prime}(s_{i},a_{i},b_i)+(\mathbb{P}_{i}^{\prime}-\mathbb{P}_{i}^{\prime\prime})V_{i+1}^{\prime}(s_{i},a_{i},b_i)]\mid s_{h}=s\right].$$
\end{lemma}

\begin{lemma}\label{lem:minmax}
    If for transition kernels $P^\prime$ and $P^{\prime\prime}$, a constant $\epsilon>0$, a reward function $r$ and policy sets $\Pi_A$, $\Pi_B$, it holds that for any $(\mu,\nu)\in\Pi_A\times\Pi_B$, $\left|V^{\mu,\nu}(r,P^\prime)-V^{\mu,\nu}(r,P^{\prime\prime})\right|\leq \epsilon$, then we have for all $\mu\in\Pi_A$,
    $$\left|\inf_{\nu\in\Pi_B}V^{\mu,\nu}(r,P^\prime)-\inf_{\nu\in\Pi_B}V^{\mu,\nu}(r,P^{\prime\prime})\right|\leq \epsilon.$$
    As a result, it holds that
    $$\left|\sup_{\mu\in\Pi_A}\inf_{\nu\in\Pi_B}V^{\mu,\nu}(r,P^\prime)-\sup_{\mu\in\Pi_A}\inf_{\nu\in\Pi_B}V^{\mu,\nu}(r,P^{\prime\prime})\right|\leq \epsilon.$$
\end{lemma}

\begin{proof}[Proof of Lemma \ref{lem:minmax}]
    The lemma directly results from the property of $\sup$ and $\inf$.
\end{proof}

\end{document}